\documentclass{article}

\usepackage[utf8]{inputenc}
\usepackage[english]{babel}
\usepackage{authblk,url}
\usepackage{amssymb,amsmath,amsthm,twoopt,xargs,mathtools,graphicx}
\usepackage{times,ifthen}
\usepackage{fancyhdr,xcolor}
\usepackage{subcaption}
\usepackage{booktabs} 
\usepackage{algorithmic,algorithm2e}
\usepackage[a4paper, total={6in, 8in}]{geometry}

\theoremstyle{plain}
\newtheorem{theorem}{Theorem}[section]

\newtheorem{lemma}[theorem]{Lemma}

\theoremstyle{definition}

\theoremstyle{remark}

\newcommand{\rset}{\mathbb{R}}
\newcommand{\latentcont}{\mathsf{z}_e}
\newcommand{\latentcontpred}{\mathsf{h}_e}
\newcommand{\latentdis}{\mathsf{z}_q}
\newcommand{\rme}{\mathrm{e}}
\newcommand{\embedspace}{\mathcal{E}}
\newcommand{\embed}{\rme}
\newcommand{\bckw}{\tilde{q}}

\title{Diffusion bridges  vector quantized Variational AutoEncoders}
\date{}

\author[$\star$, $\ddag$]{Max Cohen}
\author[$\dag$]{Guillaume Quispe}
\author[$\star$]{Sylvain Le Corff}
\author[$\dag$]{Charles Ollion}
\author[$\dag$]{\'Eric Moulines}

\affil[$\star$]{{\small SAMOVAR, T\'el\'ecom SudParis, Institut Polytechnique de Paris, Palaiseau.}}
\affil[$\ddag$]{{\small Accenta, Boulogne-Billancourt.}}
\affil[$\dag$]{{\small CMAP, \'Ecole Polytechnique, Institut Polytechnique de Paris, Palaiseau.}}

\begin{document}

\maketitle

\begin{abstract}
Vector Quantized-Variational AutoEncoders (VQ-VAE) are generative models based on discrete latent representations of the data, where inputs are mapped to a finite set of learned embeddings.
To generate new samples, an autoregressive prior distribution over the discrete states must be trained separately. This prior is generally very complex and leads to slow generation. In this work, we propose a new model to train the prior and the encoder/decoder networks simultaneously. We build a diffusion bridge between a continuous coded vector and a non-informative prior distribution.  The latent discrete states are then given as random functions of these continuous vectors. We show that our model is competitive with the autoregressive prior on the mini-Imagenet and CIFAR dataset and is efficient in both optimization and sampling. Our framework also extends the standard VQ-VAE and enables end-to-end training.
\end{abstract}

\section{Introduction}
Variational AutoEncoders (VAE) have emerged as important generative models based on latent representations of the data.
While the latent states are usually continuous vectors, Vector Quantized Variational AutoEncoders (VQ-VAE) have demonstrated the usefulness of discrete latent spaces and have been successfully applied in image and speech generation \cite{oord2017neural, esser2021taming, ramesh2021zero}. 

In a VQ-VAE, the distribution of the inputs is assumed to depend on a hidden discrete state. Large scale image generation VQ-VAEs use for instance multiple discrete latent states, typically organized as 2-dimensional lattices. In the original VQ-VAE, the authors propose a variational approach to approximate the posterior distribution of the discrete states given the observations. The variational distribution takes as input the observation, which is passed through an encoder. The discrete latent variable is then computed by a nearest neighbour procedure that maps the encoded vector to the nearest discrete embedding.

 It has been argued that the success of VQ-VAEs lies in the fact that they do not suffer from the usual posterior collapse of VAEs  \cite{oord2017neural}. However, the implementation of VQ-VAE involves many practical tricks and still suffers from several limitations. First, the quantization step leads the authors to propose a rough approximation of the gradient of the loss function by copying gradients from the decoder input to the encoder output. Second, the prior distribution of the discrete variables is initially assumed to be uniform when training the VQ-VAE. In a second training step, high-dimensional autoregressive models such as PixelCNN \cite{oord2016conditional, salimans2017pixelcnn, chen2018pixelsnail} and WaveNet \cite{oord2016wavenet} are estimated to obtain a complex prior distribution. Joint training of the prior and the VQ-VAE is a challenging task for which no satisfactory solutions exist yet. Our work addresses both problems by introducing a new mathematical framework that extends and generalizes the standard VQ-VAE. Our method enables end-to-end training and, in particular, bypasses the separate training of an autoregressive prior.

An autoregressive pixelCNN prior model has several drawbacks, which are the same in the pixel space or in the latent space.  The data is assumed to have a fixed sequential order, which forces the generation to start at a certain point, typically in the upper left corner, and span the image or the 2-dimensional latent lattice in an arbitrary way. At each step, a new latent variable is sampled using the previously sampled pixels or latent variables. Inference may then accumulate prediction errors, while training provides ground truth at each step. 
 The runtime process, which depends mainly on the number of network evaluations, is sequential and depends on the size of the image or the 2-dimensional latent lattice, which can become very large for high-dimensional objects.

 The influence of the prior is further explored in \cite{razavi2019generating}, where VQ-VAE is used to sample images on a larger scale, using two layers of discrete latent variables, and  \cite{willetts:2021} use hierarchical discrete VAEs with numerous layers of latent variables. Other works such as \cite{esser2021taming, ramesh2021zero} have used Transformers to autoregressively model a sequence of latent variables: while these works benefit from the recent advances of Transformers for large language models, their autoregressive process still suffers from the same drawbacks as pixelCNN-like priors.

 The main claim of our paper is that using diffusions in a continuous space,  $\mathbb{R}^{d \times N}$ in our setting, is a very efficient way to learn complex discrete distributions, with support on a large space (here with cardinality $K^N$). We only require an embedded space, an uninformative target distribution (here a Gaussian law), and use a continuous bridge process to learn the discrete target distribution. In that direction, our contribution is inspired by the literature but also significantly different. Our procedure departs from the diffusion probabilistic model approach of \cite{ho2020denoising}, which highlights the role of bridge processes in denoising continuous target laws, and from \cite{hoogeboom2021argmax},  where multinomial diffusions are used to noise and denoise but prevent the use of the expressiveness of continuous bridges, and also do not scale well with $K$ as remarked by its authors. Although we target a discrete distribution, our approach does not suffer from this limitation.

 Our contributions are summarized as follows.
\begin{itemize}
    \item We propose a new mathematical framework for VQ-VAEs. We introduce a two-stage prior distribution. Following the diffusion probabilistic model approach of \cite{ho2020denoising}, we consider first a continuous latent vector parameterized as a Markov chain. The discrete latent states are defined  as random functions of this Markov chain. The  transition kernels of the continuous latent variables are trained using diffusion bridges to gradually produce samples that match the data.
    \item  To our best knowledge, this is the first probabilistic generative model to use denoising diffusion in discrete latent space. This framework allows for end-to-end training of VQ-VAE.
    \item We focus on VQ-VAE as our framework enables simultaneous training of all components of those popular discrete models which is not straightforward. However, our methodology is  more general and allows the use of continuous embeddings and diffusion bridges to sample form any discrete laws.
    \item We present our method on a toy dataset and then compare its efficiency to the pixelCNN prior of the original VQ-VAE on the miniImagenet dataset.
\end{itemize}
Figure~\ref{fig:archi} describes the complete architecture of our model.
\begin{figure}[h]
    \centering
    \includegraphics[width=.8\linewidth]{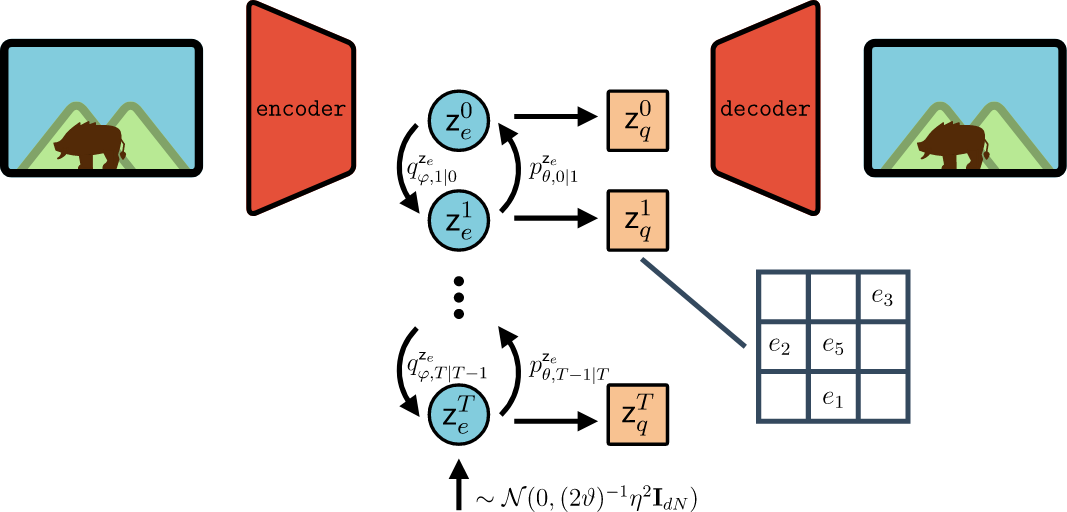}
    \caption{Our proposed architecture, for a prior based on a Ornstein-Uhlenbeck bridge. The top pathway from \textit{input image} to $\latentcont^0$, to $\latentdis^0$, to \textit{reconstructed image} resembles the original VQ-VAE model. The vertical pathway from $(\latentcont^0, \latentdis^0)$ to $(\latentcont^T, \latentdis^T)$ and backwards is based on a denoising diffusion process. See Section~\ref{sec:OU} and Algorithm~\ref{alg:sample} for the corresponding sampling procedure.}
    \label{fig:archi}
\end{figure}

\section{Related Works}

\paragraph{Diffusion Probabilistic Models.}
A promising class of models that depart from autoregressive models are Diffusion Probabilistic Models \cite{sohldickstein2015deep, ho2020denoising} and closely related Score-Matching Generative Models \cite{song2019generative, de2021simulating}. The general idea is to apply a corrupting Markovian process on the data through $T$ corrupting steps and learn a neural network that gradually \textit{denoises} or reconstructs the original samples from the noisy data.
For example, when sampling images, an initial sample is drawn from an uninformative distribution and reconstructed iteratively using the trained Markov kernel. This process is applied to all pixels simultaneously, so no fixed order is required and the sampling time does not depend on sequential predictions that depend on the number of pixels, but on the number of steps $T$. While this number of steps can be large ($T=1000$ is typical), simple improvements enable to reduce it dramatically and obtain $\times 50$ speedups \cite{song2021denoising}. These properties have led diffusion probability models to receive much attention in the context of continuous input modelling.

\paragraph{From Continuous to Discrete Generative denoising.}
In \cite{hoogeboom2021argmax}, the authors propose multinomial diffusion to gradually add categorical noise to discrete samples for which the generative denoising process is learned. Unlike alternatives such as normalizing flows, the diffusion proposed by the authors for discrete variables does not require gradient approximations because the parameter of the diffusion is fixed.

Such diffusion models are optimized using variational inference to learn the denoising process, i.e., the bridge that aims at inverting the multinomial diffusion. In \cite{hoogeboom2021argmax}, the authors propose a variational distribution based on bridge sampling.  
In \cite{austin2021structured}, the authors improve the idea by modifying the transition matrices of the corruption scheme with several tricks. The main one is the addition of absorbing states in the corruption scheme by replacing a discrete value with a MASK class, inspired by recent Masked Language Models like BERT. In this way, the corrupted dimensions can be distinguished from the original ones instead of being uniformly sampled. 
One drawback of their approach, mentioned by the authors, is that the transition matrix does not scale well for a large number of embedding vectors, which is typically the case in VQ-VAE.

Compared to discrete generative denoising, our approach takes advantage of the fact that the discrete distribution depends solely on a continuous distribution in VQ-VAE. We derive a novel model based on continuous-discrete diffusion that we believe is simpler and more scalable than the models mentioned in this section.

\paragraph{From Data to Latent Generative denoising.}
Instead of modelling the data directly, \cite{vahdat2021score} propose to perform score matching in a latent space. The authors propose a complete generative model and are able to train the encoder/decoder and score matching end-to-end. Their method also achieve excellent visual patterns and results but relies on a number of optimization heuristics necessary for stable training. In \cite{mittal2021symbolic}, the authors  have also applied such an idea in a generative music model. Instead of working in a continuous latent space, our method is specifically designed for a discrete latent space as in VQ-VAEs.

\paragraph{Using Generative denoising in discrete latent space. }
In the model proposed by \cite{gu2021vector}, the autoregressive prior is replaced by a discrete generative denoising process, which is perhaps closer to our idea. However, the authors focus more on a text-image synthesis task where the generative denoising model is traine based on an input text: it generates a set of discrete visual tokens given a sequence of text tokens. They also consider the VQ-VAE as a trained model and focus only on the generation of latent variables. This work focuses instead on deriving a full generative model with a sound probabilistic interpretation that allows it to be trained end-to-end.

\section{Diffusion bridges VQ-VAE}
 \subsection{Model and loss function}
 Assume that the distribution of the input $x\in\rset^m$ depends on a hidden discrete state $\latentdis\in\embedspace = \{\embed_1,\ldots,\embed_K\}$ with $\embed_k\in\rset^d$ for all $1\leqslant k \leqslant K$. Let $p_\theta$ be the joint probability density of $(\latentdis,x)$
 $$
 (\latentdis,x)\mapsto p_\theta(\latentdis,x) =  p_\theta(\latentdis) p_\theta(x|\latentdis)\,,
 $$
 where $\theta\in\rset^p$ are unknown parameters.
Consider first an encoding function $f_\varphi$ and write $\latentcont(x)= f_\varphi(x)$ the encoded data. In the original VQ-VAE, the authors proposed the following variational distribution to approximate $p_\theta(\latentdis|x)$:
 $$
  q_\varphi(\latentdis|x) = \delta_{\embed_{k^*_x}}(\latentdis)\,,
  $$
  where $\delta$ is the Dirac mass and
 $$
  k^*_x = \mathrm{argmin}_{1\leqslant k \leqslant K}\left\{\|\latentcont(x)-\embed_k\|_2\right\}\,,
 $$
  where $\varphi\in\rset^r$ are all the variational parameters.

In this paper, we introduce a diffusion-based generative VQ-VAE. This model allows to propose a VAE approach with an efficient joint training of the prior and the variational approximation. 
Assume that $\latentdis$ is a sequence, i.e.  $\latentdis= \latentdis^{0:T}$, where for all sequences $(a_u)_{u\geqslant 0}$ and all $0\leqslant s\leqslant t$, $a^{s:t}$ stands for $(a_s,\ldots,a_t)$. Consider the following joint probability distribution 
$$
p_{\theta}(\latentdis^{0:T},x) = p^{\latentdis}_{\theta}(\latentdis^{0:T})p^x_{\theta}(x|\latentdis^{0})\,.
$$ 
The latent discrete state $\latentdis^0$ used as input in the decoder  is the final state of the chain $(\latentdis^T,\ldots,\latentdis^0)$. We further assume that $p_{\theta}^{\latentdis}(\latentdis^{0:T})$ is the marginal distribution of  
$$
p_{\theta}(\latentdis^{0:T},\latentcont^{0:T}) 
= p^{\latentcont}_{\theta,T}(\latentcont^T) p^{\latentdis}_{\theta,T}(\latentdis^T|\latentcont^T)\prod_{t=0}^{T-1}p^{\latentcont}_{\theta,t|t+1}(\latentcont^t|\latentcont^{t+1})p^{\latentdis}_{\theta,t}(\latentdis^t|\latentcont^t)\,.
$$
In this setting, $\{\latentcont^t\}_{0\leqslant t\leqslant T}$ are continuous latent states in $\mathbb{R}^{d\times N}$ and conditionally on $\{\latentcont^t\}_{0\leqslant t\leqslant T}$ the $\{\latentdis^t\}_{0\leqslant t\leqslant T}$ are independent with discrete distribution with support $\embedspace^N$. This means that we model jointly $N$  latent states as this is useful for many applications such as image generation. 
The continuous latent state is assumed to be a Markov chain and at each time step $t$ the discrete variable $\latentdis^t$ is a random function of the corresponding $\latentcont^t$.  Although the continuous states are modeled as a Markov chain, the discrete variables arising therefrom have a more complex statistical structure (and in particular are not Markovian).

The prior distribution of  $\latentcont^T$ is assumed to be uninformative and this is the sequence of denoising transition densities $\{p^{\latentcont}_{\theta,t|t+1}\}_{0\leqslant t\leqslant T-1}$ which provides the final latent state $\latentcont^0$ which is mapped to the embedding space and used in the decoder, i.e. the conditional law of the data given the latent states. The final discrete $\latentdis^0$ only depends the continuous latent variable  $\latentcont^0$, similar to the dependency between $\latentdis$ and $\latentcont$ in the original VQ-VAE.

Since the conditional law $p_{\theta}(\latentdis^{0:T},\latentcont^{0:T}| x)$ is not available explicitly, this work focuses on  variational approaches to provide an approximation. Then, consider the following variational family:
$$
q_{\varphi}(\latentdis^{0:T},\latentcont^{0:T}| x) = \delta_{\latentcont(x)}(\latentcont^0)q_{\varphi,0}^{\latentdis}(\latentdis^0|\latentcont^0)\prod_{t=1}^T\left\{ q^{\latentcont}_{\varphi,t|t-1}(\latentcont^t|\latentcont^{t-1})q^{\latentdis}_{\varphi,t}(\latentdis^t|\latentcont^t)\right\}\,.
$$
The family $\{q^{\latentcont}_{\varphi,t|t-1}\}_{1\leqslant t \leqslant T}$  of forward "noising" transition densities are chosen to be the transition densities of a continuous-time process $(Z_t)_{t\geqslant 0}$ with $Z_0 = \latentcont(x)$. Sampling the diffusion bridge $(\tilde Z_t)_{t\geqslant 0}$, i.e. the law of the process $(Z_t)_{t\geqslant 0}$  conditioned on $Z_0 = \latentcont(x)$ and $Z_T = \latentcont^T$ is a challenging problem for general diffusions, see for instance \cite{beskos2008mcmc,lin2010generating,bladt2016simulation}. By the Markov property, the  marginal density at time $t$ of this conditioned process is given by:
\begin{equation}
\label{eq:markov:bridge}
\bckw^{\latentcont}_{\varphi,t|0,T}(\latentcont^t|\latentcont^0,\latentcont^T) = \frac{q^{\latentcont}_{\varphi,t|0}(\latentcont^t|\latentcont^{0})q^{\latentcont}_{\varphi,T|t}(\latentcont^T|\latentcont^{t})}{q^{\latentcont}_{\varphi,T|0}(\latentcont^T|\latentcont^{0})}\,.
\end{equation}
The Evidence Lower BOund (ELBO) is then defined, for all $(\theta,\varphi)$, as
$$
\mathcal{L}(\theta,\varphi) = \mathbb{E}_{q_{\varphi}}\left[\log \frac{p_{\theta}(\latentdis^{0:T},\latentcont^{0:T},x)}{q_{\varphi}(\latentdis^{0:T},\latentcont^{0:T}| x)}\right]\,,
$$
where $\mathbb{E}_{q_{\varphi}}$ is the expectation under $q_{\varphi}(\latentdis^{0:T},\latentcont^{0:T}| x)$.
\begin{lemma}
\label{lem:loss}
For all $(\theta,\varphi)$, the ELBO $\mathcal{L}(\theta,\varphi)$ is:
$$
\mathcal{L}(\theta,\varphi) = \mathbb{E}_{q_{\varphi}}\left[\log p^x_{\theta}(x|\latentdis^{0})\right] + \sum_{t=0}^T \mathcal{L}_t(\theta,\varphi)+ \sum_{t=0}^T\mathbb{E}_{q_{\varphi}}\left[\log \frac{p_{\theta,t}^{\latentdis}(\latentdis^{t}|\latentcont^{t})}{q_{\varphi,t}^{\latentdis}(\latentdis^{t}|\latentcont^{t})}\right]\,, 
$$
where, for $1\leqslant t \leqslant T-1$,
\begin{align*}
\mathcal{L}_0(\theta,\varphi) &=  \mathbb{E}_{q_\varphi}\left[\log p^{\latentcont}_{\theta, 0|1}(\latentcont^0|\latentcont^{1})\right]\,,\\
\mathcal{L}_t(\theta,\varphi) &= \mathbb{E}_{q_{\varphi}}\left[\log \frac{p_{\theta,t-1|t}^{\latentcont}(\latentcont^{t-1}|\latentcont^{t})}{q^{\latentcont}_{\varphi,t-1|0,t}(\latentcont^{t-1}|\latentcont^{0},\latentcont^{t})}\right]\,,\\
    \mathcal{L}_T(\theta,\varphi) &= \mathbb{E}_{q_{\varphi}}\left[\log \frac{p^{\latentcont}_{\theta,T}(\latentcont^{T})}{q_{\varphi,T|0}^{\latentcont}(\latentcont^{T}|\latentcont^{0})}\right]\,.
\end{align*}
\end{lemma}
\begin{proof}
The proof is standard and postponed to Appendix~\ref{ap:loss}.
\end{proof}
The three terms of the objective function can be interpreted as follows:
$$
\mathcal{L}(\theta,\varphi) = \mathcal{L}^{rec}(\theta,\varphi) + \sum_{t=0}^T \mathcal{L}_t(\theta,\varphi) + \sum_{t=0}^T \mathcal{L}^{reg}_t(\theta,\varphi)
$$
with $\mathcal{L}^{rec} = \mathbb{E}_{q_{\varphi}}[\log p^x_{\theta}(x|\latentdis^{0})]$ a reconstruction term, $\mathcal{L}_t$ the diffusion term, and an extra term \begin{equation}
\mathcal{L}^{reg}_t = \mathbb{E}_{q_{\varphi}}\left[\log \frac{p_{\theta,t}^{\latentdis}(\latentdis^{t}|\latentcont^{t})}{q_{\varphi,t}^{\latentdis}(\latentdis^{t}|\latentcont^{t})}\right]\,,
\end{equation}
which may be seen as a regularization term as discussed in next sections.
\subsection{Application to Ornstein-Uhlenbeck processes}
\label{sec:OU}
Consider for instance the following Stochastic Differential Equation (SDE) to add noise to the  normalized inputs:
\begin{equation}
\label{eq:ou}
\mathrm{d}Z_t = -\vartheta (Z_t - z_*)\mathrm{d}t + \eta\mathrm{d}W_t\,,
\end{equation}
where $\vartheta, \eta>0$,  $z_*\in\mathbb{R}^{d\times N}$ is the target state at the end of the noising process and $\{W_t\}_{0\leqslant t\leqslant T}$ is a standard Brownian motion in $\mathbb{R}^{d\times N}$. We can define the variational density by integrating this SDE along small step-sizes. Let $\delta_t$ be the time step between the two consecutive latent variables $\latentcont^{t-1}$ and $\latentcont^{t}$. In this setting, $q^{\latentcont}_{\varphi,t|t-1}(\latentcont^t|\latentcont^{t-1})$ is a Gaussian probability density function with mean $z_* + (\latentcont^{t-1}-z_*)\mathrm{e}^{-\vartheta \delta_t}$ in $\mathbb{R}^{d\times N}$ and covariance matrix $(2\vartheta)^{-1}\eta^2(1-\mathrm{e}^{-2\vartheta\delta_t})\mathbf{I}_{dN}$, where for all $n\geqslant 1$, $\mathbf{I}_{n}$ is the identity matrix with size $n\times n$. Asymptotically the process is a Gaussian with mean $z_*$ and variance $\eta^2(2\vartheta)^{-1} \mathbf{I}_{dN}$.

The denoising process amounts then to sampling from the bridge associated with the SDE, i.e. sampling $\latentcont^{t-1}$ given $\latentcont^0$ and $\latentcont^t$. The law of this bridge is explicit for the Ornstein-Uhlenbeck diffusion \eqref{eq:ou}.
Using \eqref{eq:markov:bridge},
$$
\bckw^{\latentcont}_{\varphi,s|0,t}(\latentcont^{s}|\latentcont^{t},\latentcont^{0}) \propto q^{\latentcont}_{\varphi,s|0}(\latentcont^{t-1}|\latentcont^{0}) q^{\latentcont}_{\varphi,t|s}(\latentcont^{t}|\latentcont^{s})\,,
$$
where $0\leqslant s\leqslant t$, so that $\bckw^{\latentcont}_{\varphi,t-1|0,t}(\latentcont^{t-1}|\latentcont^{t},\latentcont^{0})$ is a Gaussian probability density function with mean
$$
 \tilde \mu_{\varphi,t-1|0,t}(\latentcont^0,\latentcont^t) = \frac{\beta_t}{1-\bar{\alpha}_t}\left(z_* + \sqrt{\bar{\alpha}_{t-1}}(\latentcont^0-z_*)\right) + \frac{1-\bar{\alpha}_{t-1}}{1-\bar{\alpha}_t}\sqrt{\alpha}_t\left(\latentcont^t - (1 - \sqrt{\alpha_t} )z_*\right)
$$
 and covariance matrix
 $$
 \tilde \sigma^2_{\varphi,t-1|0,t} = \frac{\eta^2}{2\vartheta}\frac{1- \bar \alpha_{t-1}}{1- \bar \alpha_{t}}\beta_t\, \mathbf{I}_{dN}\,,
 $$
 where $\beta_t = 1 - \mathrm{exp}(-2\vartheta \delta_t)$, $\alpha_t = 1-\beta_t$ and $\bar{\alpha}_t = \prod_{s=1}^{t}\alpha_s$.
 Note that the bridge sampler proposed in \cite{ho2020denoising} is a specific case of this setting with $\eta = \sqrt{2}$, $z_*=0$ and $\vartheta = 1$. 

\paragraph{Choice of denoising model $p_\theta$.}
Following \cite{ho2020denoising}, we propose a Gaussian distribution for $p_{\theta,t-1|t}^{\latentcont}(\latentcont^{t-1}|\latentcont^{t})$ with mean $\mu_{\theta,t-1|t}(\latentcont^{t},t)$ and variance $\sigma_{\theta,t|t-1}^2\,\mathbf{I}_{dN}$. In the following, we choose $$
\sigma_{\theta,t|t-1}^2 = \frac{\eta^2}{2\vartheta}\frac{1- \bar \alpha_{t-1}}{1- \bar \alpha_{t}}\beta_t\,
$$
so that the term $\mathcal{L}_t$ of Lemma~\ref{lem:loss} writes
$$
    2\sigma_{\theta,t|t-1}^2\mathcal{L}_t(\theta,\varphi)   = -\mathbb{E}_{q_\varphi}\left[\left\|\mu_{\theta,t-1|t}(\latentcont^{t},t) -  \tilde \mu_{\varphi,t-1|0,t}(\latentcont^0,\latentcont^t)\right\|_2^2\right]\,.
$$
In addition, under $q_\varphi$, $\latentcont^t$ has the same distribution as
$$
\latentcontpred^t(\latentcont^{0},\varepsilon_t) = z_* + \sqrt{\bar \alpha_t}(\latentcont^{0}-z_*) + \sqrt{\frac{\eta^2}{2\vartheta}(1-\bar\alpha_t)}\varepsilon_t\,,
$$
where $\varepsilon_t \sim \mathcal{N}(0,\mathbf{I}_{dN})$. Then, for instance in the case $z_*=0$, $\tilde \mu_{\varphi,t-1|0,t}$ can be reparameterised as follows:
$$
    \tilde \mu_{\varphi,t-1|0,t}(\latentcont^0,\latentcont^t) =  \frac{1}{\sqrt{\alpha_t}}\left(\latentcontpred^t(\latentcont^0,\varepsilon_t) - \sqrt{\frac{\eta^2}{2\vartheta (1-\bar{\alpha}_t)}}\beta_t\varepsilon_t\right)\,.
$$
We therefore propose to use
$$
    \mu_{\theta,t-1|t}(\latentcont^{t},t) =
     \frac{1}{\sqrt{\alpha_t}}\left(\latentcont^t - \sqrt{\frac{\eta^2}{2\vartheta (1-\bar{\alpha}_t)}}\beta_t\varepsilon_\theta(\latentcont^{t},t)\right)\,,
$$
which yields
\begin{equation}
    \label{eq:hoparam}
    \mathcal{L}_t(\theta,\varphi) \\
    = \frac{-\beta_t}{2\alpha_t(1-\bar\alpha_{t-1})}\mathbb{E}\left[\left\|\varepsilon_t - \varepsilon_\theta(\latentcontpred^t(\latentcont^{0},\varepsilon_t) ,t) \right\|_2^2\right]\,.
\end{equation}
Several choices can be proposed to model the function $\varepsilon_\theta$. The deep learning architectures considered in the numerical experiments are discussed in Appendix~\ref{ap:networks} and ~\ref{ap:additionaltoy}. Similarly to \cite{ho2020denoising}, we use a stochastic version of our loss function:  sample $t$ uniformly in $\{0, \ldots,  T\}$, and consider $\mathcal{L}_t(\theta,\varphi)$ instead of the full sum over all $t$. The final training algorithm is described in Algorithm~\ref{alg:train} and the sampling procedure in Algorithm~\ref{alg:sample}. 

\paragraph{Connections with the VQ-VAE loss function. } In the special case where $T=0$, our loss function can be reduced to a standard VQ-VAE loss function. In that case, write $\latentdis = \latentdis^0$ and $\latentcont = \latentcont^0$, the ELBO then becomes: 
$$
\mathcal{L}(\theta,\varphi) = \mathbb{E}_{q_{\varphi}}\left[\log p^x_{\theta}(x|\latentdis)\right]
+ \mathbb{E}_{q_{\varphi}}\left[\log \frac{p_{\theta}^{\latentdis}(\latentdis|\latentcont)}{q_{\varphi}^{\latentdis}(\latentdis|\latentcont)}\right]\,, 
$$
Then, if we assume that $p_{\theta}^{\latentdis}(\latentdis|\latentcont) = \mathrm{Softmax}\{-\|\latentcont - \embed_k\|^2_2\}_{1\leq k \leq K}$ and that  $q_{\varphi}^{\latentdis}(\latentdis|\latentcont)$ is as in \cite{oord2017neural}, i.e. a Dirac mass at $\widehat{\latentdis} = \mathrm{argmin}_{1\leq k \leq K}\|\latentcont - \embed_k\|^2_2$, up to an additive constant, this yields the following random estimation of $\mathbb{E}_{q_{\varphi}}[\log p_{\theta}^{\latentdis}(\latentdis|\latentcont)/q_{\varphi}^{\latentdis}(\latentdis|\latentcont)]$,
$$
    \widehat{\mathcal{L}}^{reg}_{\latentdis}(\theta,\varphi) = \|\latentcont - \widehat{\latentdis}\|_2 + \log \left(\sum_{k=1}^{K}\exp\left\{-\|\latentcont-\embed_k\|_2\right\}\right)\,.
$$
The first term of this loss is the loss proposed in \cite{oord2017neural} which is then split into two parts using the stop gradient operator. The last term is simply the additional normalizing term of  $p_{\theta}^{\latentdis}(\latentdis|\latentcont)$. 

\paragraph{Connecting diffusion and discretisation. } Similar to the VQ-VAE case above, it is possible to consider only the term $\mathcal{L}^{reg}_0(\theta,\varphi)$ in the case $T > 0$. However, our framework allows for much flexible parameterisation of $p_{\theta,t}^{\latentdis}(\latentdis^t|\latentcont^t)$ and $q_{\varphi,t}^{\latentdis}(\latentdis^t|\latentcont^t)$. For instance, the Gumbel-Softmax trick provides an efficient and differentiable parameterisation. A sample $\latentdis^t\sim  p_{\theta,t}^{\latentdis}(\latentdis^t|\latentcont^t)$ (resp. $\latentdis^t\sim q_{\varphi,t}^{\latentdis}(\latentdis^t|\latentcont^t)$) can be obtained by sampling with probabilities proportional to $\{\exp\{(-\|\latentcont - \embed_k\|^2_2 + G_k )/\tau_t\}\}_{1\leq k \leq K}$ (resp. $\{\exp\{(-\|\latentcont - \embed_k\|^2_2 + \tilde G_k )/\tau\}\}_{1\leq k \leq K}$), where $\{(G_k,\tilde G_k)\}_{1\leq k \leq K}$ are i.i.d. with distribution $\mathrm{Gumbel}(0,1)$, $\tau>0$, and $\{\tau_t\}_{0\leq t \leq T}$ are positive  time-dependent scaling parameters. In practice, the third part of the objective function can be computed efficiently, by using a stochastic version of the ELBO, computing a single $\mathcal{L}^{reg}_t(\theta,\varphi)$ instead of the sum (we use the same $t$ for both parts of the ELBO). The term reduces to:

\begin{equation}
    \mathcal{L}^{reg}_t(\theta,\varphi) = -\mathrm{KL}(q_\varphi(\latentdis^{t}|\latentcont^{t})\|p_\theta(\latentdis^{t}|\latentcont^{t}))\,. 
\end{equation}
This terms connects the diffusion and quantisation parts as it creates a gradient pathway through a step $t$ of the diffusion process, acting as a regularisation on the codebooks and $\latentcont^t$. Intuitively, maximizing $\mathcal{L}^{reg}_t(\theta,\varphi)$  accounts for pushing codebooks and $\latentcont^t$ together or apart depending on the choice of $\tau, \tau_t$. The final end-to-end training algorithm is described in Algorithm~\ref{alg:train}, and further considerations are provided  in Appendix~\ref{ap:reg}.

\begin{algorithm}[tb]
   \caption{Training procedure}
   \label{alg:train}
\begin{algorithmic}
   \REPEAT

   \STATE Compute $\latentcont^0= f_\varphi(x)$ 
   \STATE Sample $\hat{\latentdis}^0 \sim q_\varphi(\latentdis^0|\latentcont^0)$
   \STATE Compute $\textcolor{teal}{\hat{\mathcal{L}}^{rec}(\theta,\varphi)} = \log p^x_{\theta}(x|\hat{\latentdis}^0)$ 

   \STATE Sample $t \sim Uniform(\{0,\ldots, T\})$ 
   \STATE Sample $\varepsilon_t \sim \mathcal{N}(0,\mathbf{I}_{dN})$
   \STATE Sample $\latentcont^{t} \sim q_{\varphi,t}(\latentcont^t|\latentcont^0)$ (using $\varepsilon_t$)
   \STATE Compute $\textcolor{olive}{\hat{\mathcal{L}}_t(\theta,\varphi)}$ from $\varepsilon_\theta(\latentcont^{t},t)$ and $\varepsilon_t$ using \eqref{eq:hoparam}
   \STATE Compute $\textcolor{violet}{\hat{\mathcal{L}}^{reg}_t(\theta,\varphi)}$ from $\latentcont^{t}$ \textit{(see text)}

   \STATE $\hat{\mathcal{L}}(\theta,\varphi) = \textcolor{teal}{\hat{\mathcal{L}}^{rec}(\theta,\varphi)} + \textcolor{olive}{\hat{\mathcal{L}}_t(\theta,\varphi)} + \textcolor{violet}{\hat{\mathcal{L}}^{reg}_t(\theta,\varphi)}$
   \STATE Perform SGD step on $-\hat{\mathcal{L}}(\theta,\varphi)$
   \UNTIL{convergence}
\end{algorithmic}
\end{algorithm}

\begin{algorithm}[tb]
   \caption{Sampling procedure (for $z_* = 0$)}
   \label{alg:sample}
\begin{algorithmic}
   
   \STATE Sample $\latentcont^T \sim \mathcal{N}(0, (2\vartheta)^{-1}\eta^2 \mathbf{I}_{dN})$ 
   \FOR{$t=T$ {\bfseries to} $1$}
   \STATE Set $\latentcont^{t-1} = \alpha_t^{-1/2}\left(\latentcont^t - \sqrt{\frac{\eta^2}{2\vartheta (1-\bar{\alpha}_t)}}\beta_t\varepsilon_\theta(\latentcont^{t},t)\right)$
   \ENDFOR
   \STATE  Sample $\latentdis^0 \sim  p^{\latentdis}_{\theta,0}(\latentdis^0|\latentcont^0)$ \COMMENT{\textit{quantisation}}
   \STATE Sample $x ~ \sim  p^x_{\theta}(x|\latentdis^0)$ \COMMENT{\textit{decoder}}
\end{algorithmic}
\end{algorithm}

\section{Experiments}

\subsection{Toy Experiment}
In order to understand the proposed denoising procedure for VQ-VAE, consider a simple toy setting in which there is no encoder nor decoder, and the codebooks $\{\embed_j\}_{0\leqslant j \leqslant K-1}$ are fixed. In this case, with $d=2$ and $N=5$, $x = \latentcont^0 \in \rset^{2\times 5}$. We choose $K=8$ and the codebooks $\embed_j = \mu_j \in \rset^2$, $0\leqslant j \leqslant K-1$, are fixed centers at regular angular intervals in $\rset^2$ and shown in Figure~\ref{fig:toydata}; the latent states $(\latentdis^t)_{1\leq t\leq T}$ lie in $\{\embed_0,\ldots,\embed_7\}^5$. Data generation proceeds as follows. First, sample a sequence of $(q_1,\ldots,q_5)$ in $\{0,\dots,7\}$: $q_1$ has a uniform distribution, and, for $s\in\{0,1,2,3\}$, $q_{s+1} = q_s + b_s \mod 8$, where $b_s$ are independent Bernoulli samples with parameter $1/2$ taking values in $\{-1, 1\}$. Conditionally on $(q_1,\ldots,q_5)$, $x$ is a Gaussian random vector with mean $(\embed_{q_1},\ldots,\embed_{q_5})$ and variance $\mathbf{I}_{2\times 5}$.

\begin{figure}[h!]
    \centering
    \includegraphics[scale=0.55]{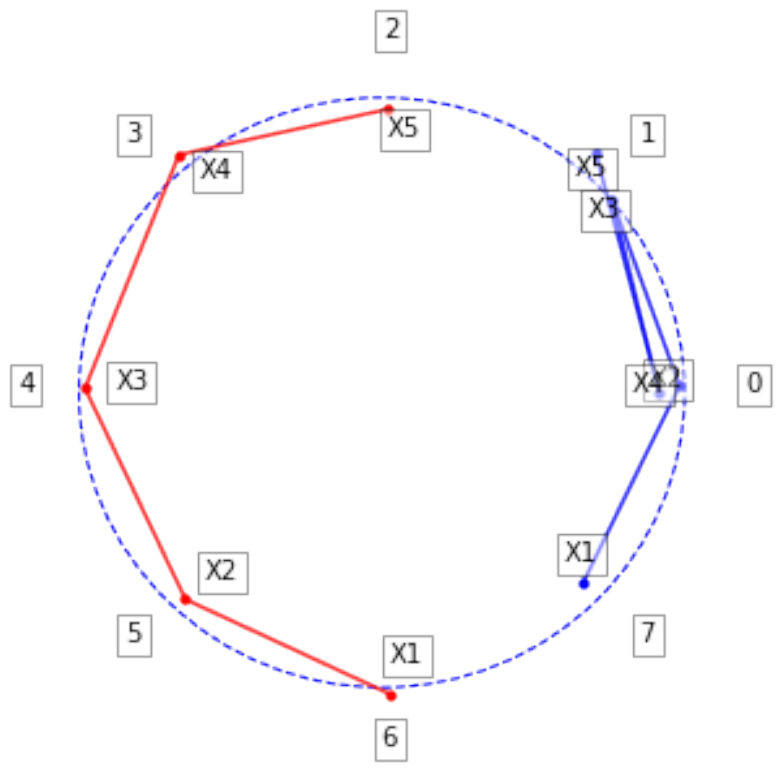}
    \caption{Toy dataset, with $K=8$ centroids, and two samples $x = (x_1,x_2,x_3,x_4,x_5)$ in $\rset^{2 \times 5}$ each displayed as $5$ points in $\rset^{2}$ (blue and red points), corresponding to the discrete sequences (red) $(6,5,4,3,2)$ and (blue) $(7,0,1,0,1)$.}
    \label{fig:toydata}
\end{figure}

We train our bridge procedure with $T=50$ timesteps, $\vartheta=2, \eta=0.1$, other architecture details and the neural network $\varepsilon_\theta(\latentcont^t ,t)$ are described in Appendix~\ref{ap:additionaltoy}. Forward noise process and denoising using $\varepsilon_\theta(\latentcont^t ,t)$ are showcased in Figure~\ref{fig:noisedenoise}, and more illustrations and experiments can be found in Appendix~\ref{ap:additionaltoy}. 

\begin{figure}[h!]
    \centering
    \includegraphics[scale=0.55]{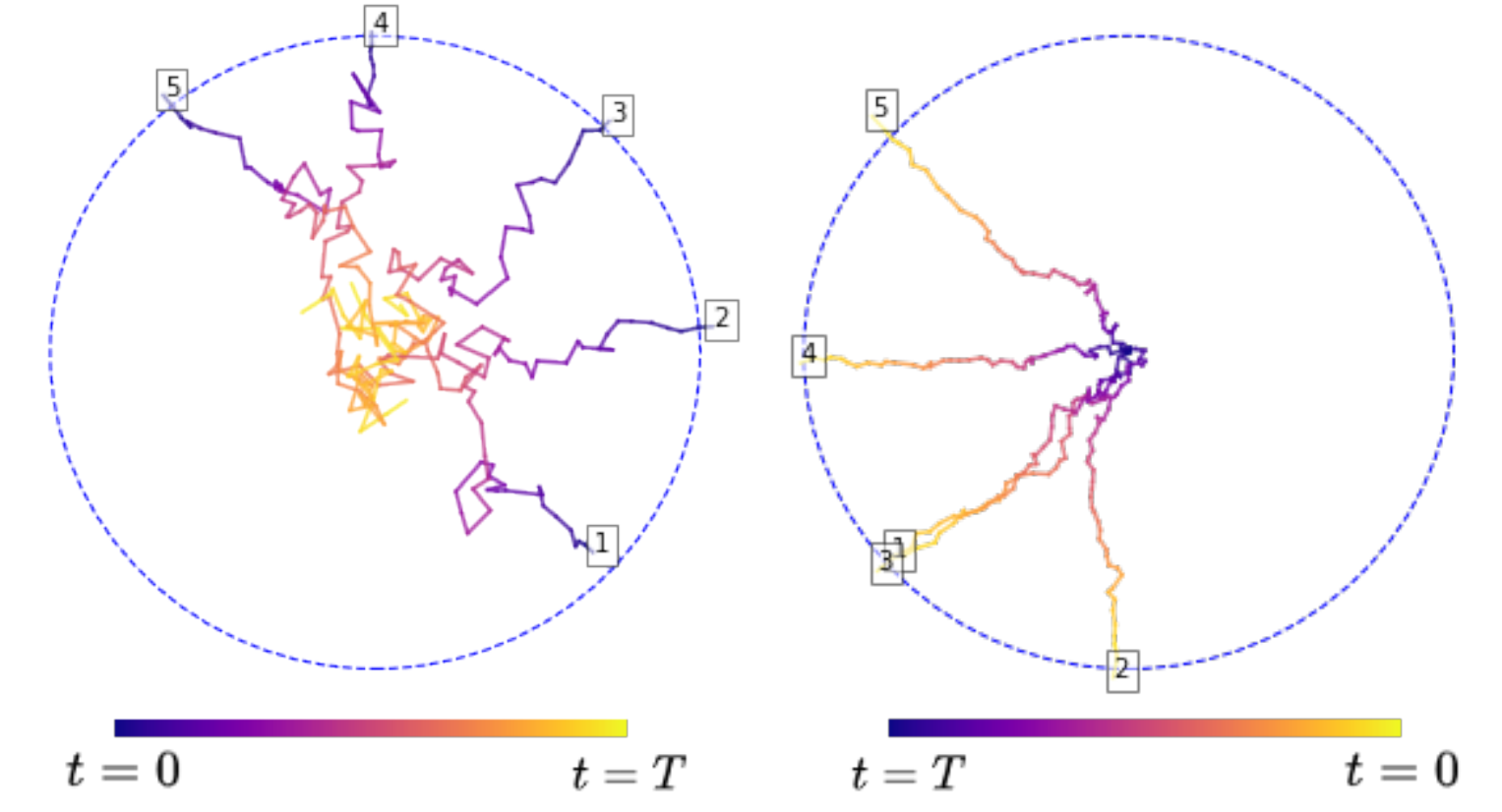}
    \caption{(Left) Forward noise process for one sample. First, one data is drawn ($\latentcont^0(x) =x$ in the toy example)  and then $\{\latentcont^t\}_{1\leq t \leq T}$ are sampled under $q_\varphi$ and displayed. 
    (Right) Reverse process for one sample $\latentcont^T\sim  \mathcal{N}(0, (2\vartheta)^{-1}\eta^2 \mathbf{I}_{dN})$. As expected, the last sample $\latentcont^0$ reaches the neighborhood of $5$ codebooks.}
    \label{fig:noisedenoise}
\end{figure}

\paragraph{End-to-end training. } Contrary to VQ-VAE procedures in which the encoder, decoder and codebooks are trained separately from the prior, we can train the bridge prior alongside the codebooks. Consider a new setup, in which the $K=8$ codebooks are randomly initialized and considered as parameters of our model (they are no longer fixed to the centers of the data generation process $\mu_j$). The first part of our loss function, in conjunction with the Gumbel-Softmax trick makes it possible to train all the parameters of the model end-to-end. Details of the procedure and results are shown in Appendix~\ref{ap:end2end}.

\subsection{Image Synthesis}
In this section, we focus on image synthesis using CIFAR10 and miniImageNet datasets. The goal is to evaluate the efficiency and properties of our model compared to the original PixelCNN. Note that for fair comparisons, the encoder, decoder and codebooks are pretrained and fixed for all models, only the prior is trained and evaluated here. As our goal is the comparison of priors, we did not focus on building the most efficient VQ-VAE, but rather a reasonable model in terms of size and efficiency.

\paragraph{CIFAR10. }
The CIFAR dataset consists of inputs $x$ of dimensions $32 \times 32$ with 3 channels. The encoder projects the input into a grid of continuous values $\latentcont^0$ of dimension $8 \times 8 \times 128$. After discretisation, $\{\latentdis^t\}_{0\leqslant t\leqslant T}$ are in a discrete latent space induced by the VQ-VAE which consists of values in $\{1,\ldots,K\}^{8 \times 8}$ with $K=256$. The pre-trained VQ-VAE reconstructions can be seen in Figure~\ref{fig:cifar_vqvae} in Appendix~\ref{ap:additional_visuals}.

\paragraph{miniImageNet. }
\textit{mini}ImageNet was introduced by \cite{Vinyals2016MatchingNF} to offer more complexity than CIFAR10, while still fitting in memory of modern machines.
600 images were sampled for 100 different classes from the original ImageNet dataset, then scaled down, to obtain 60,000 images of dimension $84 \times 84$.
In our experiments, we trained a VQVAE model to project those input images into a grid of continuous values $\latentcont^0$ of dimensions $21 \times 21 \times 32$, see Figure~\ref{fig:miniimagenet_vqvae} in Appendix~\ref{ap:additional_visuals}.
The associated codebook contains $K=128$ vectors of dimension $32$.

\paragraph{Prior models. }
Once the VQ-VAE is trained on the miniImageNet and CIFAR datasets, the $84\times 84 \times 3$ and $32\times 32 \times 3$ images respectively are passed to the encoder and result in $21 \times 21$ and $8 \times 8$ feature maps respectively. From this model, we extract the discrete latent states from training samples to train a PixelCNN prior and the continuous latent states for our diffusion.
Concerning our diffusion prior, we choose the Ornstein-Uhlenbeck process setting $\eta = \sqrt{2}$, $z_*=0$ and $\vartheta = 1$, with $T=1000$.

\paragraph{End-to-End Training.}
As an additional experiment, we propose an End-to-End training of the VQ-VAE and the diffusion process. To speed up training, we first start by pretraining the VQ-VAE, then learn the parameters of our diffusion prior alongside all the VQ-VAE parameters (encoder, decoder and codebooks). Note that in this setup, we cannot directly compare the NLL to PixelCNN or our previous diffusion model as the VQ-VAE has changed, but we can compare image generation metrics such as FID and sample quality.

\subsection{Quantitative results}
We benchmarked our model using three metrics, in order to highlight the performances of the proposed prior, the quality of produced samples as well as the associated computation costs.
Results are given as a comparison to the original PixelCNN prior for both the \textit{mini}ImageNet (see Table \ref{tab:miniimagenet}) and the CIFAR10 (see Table~\ref{tab:cifar}) datasets.

\paragraph{Negative Log Likelihood. }\label{nll_paragraph}
Unlike most related papers, we are interested in computing the Negative Log Likelihood (NLL) directly in the latent space, as to evaluate the capacity of the priors to generate coherent latent maps.
To this end, we mask a patch of the original latent space, and reconstruct the missing part, similar to image inpainting, following for instance \cite{van2016pixel}.
 In the case of our prior, for each sample $x$, we mask an area of the continuous latent state $\latentcont^0$, i.e. we mask some components of $\latentcont^0$, and aim at sampling the missing components given the observed ones using the prior model. Let $\underline{\latentdis}^0$ and $\underline{\latentcont}^0$ (resp. $\overline{\latentdis}^0$ and $\overline{\latentcont}^0$) be the masked (resp. observed) discrete and continuous latent variables. The target conditional likelihood is
\begin{align*}
p_{\theta}(\underline{\latentdis}^0|\overline{\latentcont}^0) &= \int p_{\theta}(\underline{\latentdis}^0,\underline{\latentcont}^0|\overline{\latentcont}^0)\mathrm{d} \underline{\latentcont}^0\,,\\
&= \int p_{\theta}(\underline{\latentdis}^0|\underline{\latentcont}^0)p_{\theta}(\underline{\latentcont}^0|\overline{\latentcont}^0)\mathrm{d} \underline{\latentcont}^0\,.
\end{align*}
This likelihood is intractable and replaced by a simple Monte Carlo estimate $\hat{p}_{\theta}(\underline{\latentdis}^0|\underline{\latentcont}^0)$ where $\underline{\latentcont}^0\sim p_{\theta}(\underline{\latentcont}^0|\overline{\latentcont}^0)$. Note that conditionally on $\underline{\latentcont}^0$ the components of $\underline{\latentdis}^0$ are assumed to be independent but $\underline{\latentcont}^0$ are sampled jointly under $p_{\theta}(\underline{\latentcont}^0|\overline{\latentcont}^0)$. As there are no continuous latent data in PixelCNN, $p_{\theta}(\underline{\latentdis}^0|\overline{\latentdis}^0)$ can be directly evaluated.

\paragraph{Fr\'echet Inception Distance. }
We report Fr\'echet Inception Distance (FID) scores by sampling a latent discrete state $\latentdis \in \embedspace^N$ from the prior, and computing the associated image through the VQ-VAE decoder. In order to evaluate each prior independently from the encoder and decoder networks, these samples are compared to VQ-VAE reconstructions of the dataset images.

\paragraph{Kullback-Leibler divergence. }
In this experiment,  we draw $M=1000$ samples from test set and encode them using the trained VQ-VAE, and then draw as many samples from the pixelCNN prior, and our diffusion prior. We propose then to compute the empirical Kullback Leibler (KL) divergence between original and sampled distribution at each pixel. Figure~\ref{fig:klmap} highlights that PixelCNN performs poorly on the latest pixels (at the bottom) while our method remains consistent. This is explained by our denoising process in the continuous space which uses all pixels jointly while  PixelCNN is based on an autoregressive model. 
\begin{figure}[htpb]
    \centering
    \includegraphics[width=0.49\textwidth]{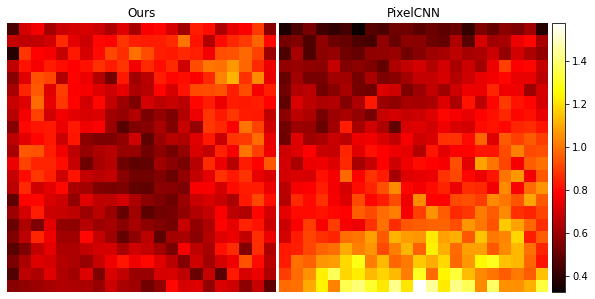}
    \caption{KL Distance between the true empirical distribution and both prior distributions in the latent space. Darker squares indicates lower (better) values.}
    \label{fig:klmap}
    \vspace{-2mm}
\end{figure}

\begin{table}[htpb]
    \centering
    \begin{tabular}{c c}
        \toprule
         & KL \\
        \toprule
        Ours & {\bf 0.713}\\
        PixelCNN & 0.809 \\
        \bottomrule
    \end{tabular}
    \caption{Averaged KL metric on the feature map.}
    \label{tab:metrics}
    \vspace{-4mm}
\end{table}

\paragraph{Computation times. }
We evaluated the computation cost of sampling a batch of 32 images, on a GTX TITAN Xp GPU card. Note that the computational bottleneck of our model consists of the $T=1000$ sequential diffusion steps (rather than the encoder/decoder which are very fast in comparison). Therefore, a diffusion speeding technique such as the one described in \cite{song2021denoising} would be straightforward to apply and would likely provide a $\times 50$ speedup as mentioned in the paper.

\begin{table}
\centering
    \caption{Results on \textit{mini}ImageNet. Metrics are computed on the validation dataset. The means are displayed along with the standard deviation in parenthesis.}
    \resizebox{.65\textwidth}{!}{\begin{tabular}{ |l||c|c|c| }
        \hline
        & NLL & FID & s/sample \\
        \hline
        PixelCNN \cite{oord2017neural}     & 1.00 ($\pm 0.05$)& 98 & 10.6s ($\pm 28ms$) \\
        Ours                               & 0.94 ($\pm 0.02$)& 99 & 1.7s ($\pm 10ms$)\\
        \hline
    \end{tabular}}
    \label{tab:miniimagenet}
\end{table}

\begin{table}
\centering
    \caption{Results on CIFAR10. Metrics are computed on the validation dataset. The means are displayed along with the standard deviation in parenthesis.}
    \resizebox{.65\textwidth}{!}{\begin{tabular}{ |l||c|c|c| }
        \hline
        & NLL & FID & s/sample \\
        \hline
        PixelCNN \cite{oord2017neural}     & 1.41 ($\pm 0.06)$ & 109 & 0.21 ($\pm 0.8ms$) \\
        Ours                               & 1.33 ($\pm 0.18$)  &  104 & 0.05s ($\pm 0.5ms$) \\
        Ours  end-to-end               & 1.59 ($\pm 0.27$)\footnote{NLL for end-to-end takes into account the full model including the modified VQ-VAE, and therefore is not directly comparable to the two others.}  &  92 & 0.11s ($\pm 0.5ms$) \\
        \hline
    \end{tabular}}
    \label{tab:cifar}
\end{table}

\subsection{Qualitative results}
\paragraph{Sampling from the prior. }
Samples from the PixelCNN prior are shown in Figure \ref{fig:pixel_samples}
and samples from our prior in Figure \ref{fig:diffusion_samples}. Additional samples are given in Appendix~\ref{ap:additional_visuals}. Note that contrary to original VQ-VAE prior, the prior is not conditioned on a class, which makes the generation less specific and more difficult. However, the produced samples illustrate that our prior can generate a wide variety of images which show a large-scale spatial coherence in comparison with samples from PixelCNN.

\begin{figure}[h!]

\begin{subfigure}{\linewidth}
    \centering
    \includegraphics[width=.7\linewidth]{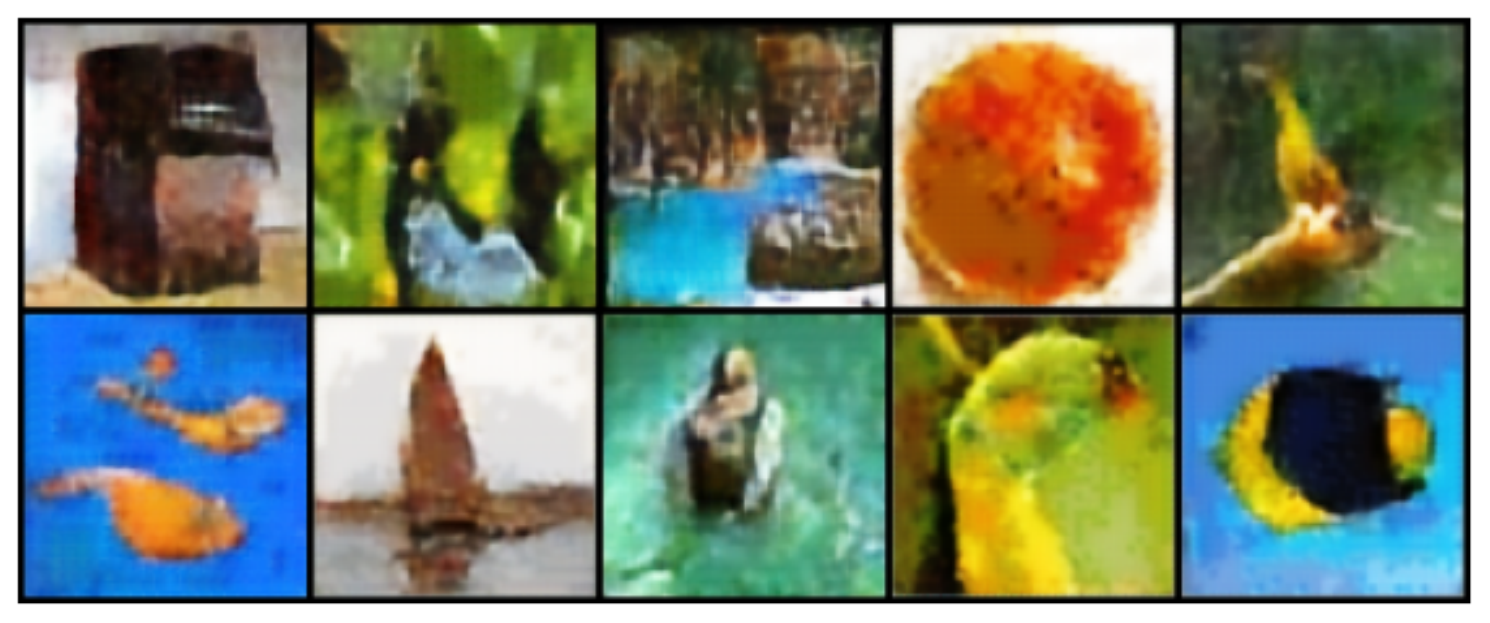}
    \caption{Samples from our diffusion prior.}
    \label{fig:diffusion_samples}
\end{subfigure}
\begin{subfigure}{\linewidth}
    \centering
    \includegraphics[width=.7\linewidth]{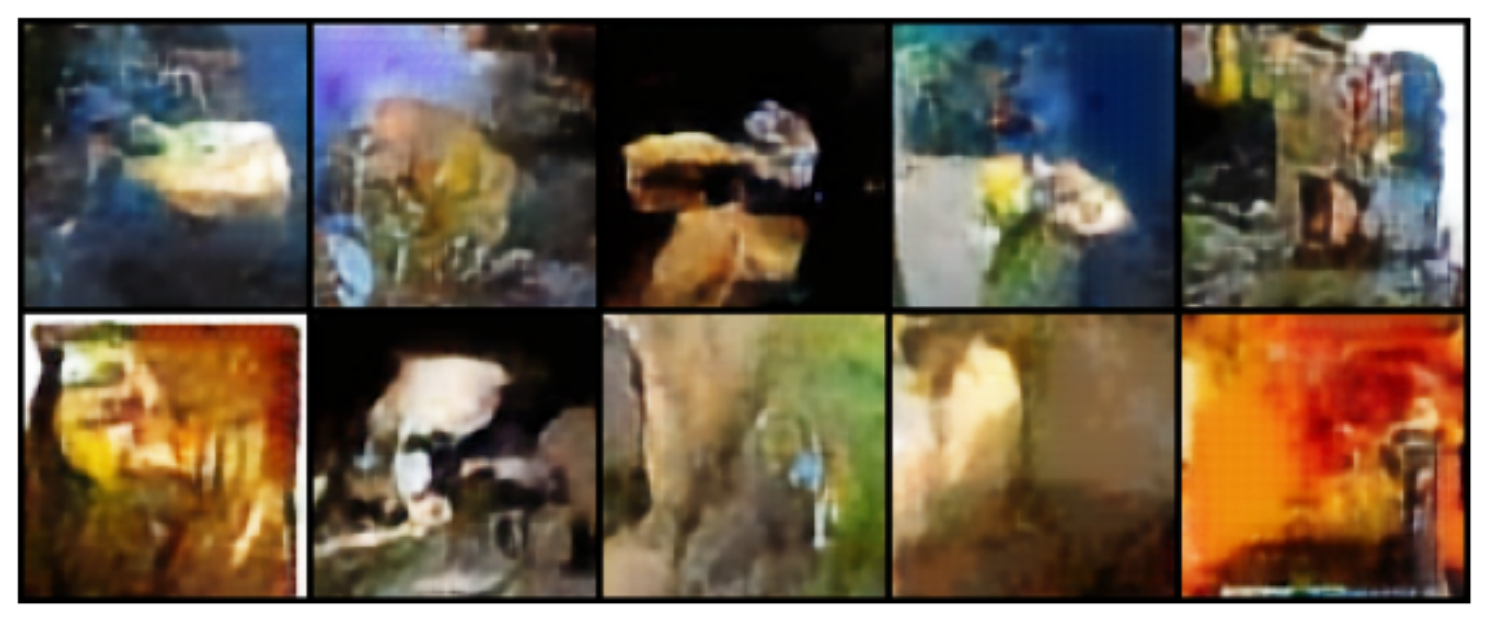}
    \caption{Samples from the PixelCNN prior.}
    \label{fig:pixel_samples}
\end{subfigure}
\caption{Comparison between samples from our diffusion-based prior (top) and PixelCNN prior (bottom).}
\label{fig:figures}
\end{figure}

\paragraph{Conditional sampling. }
As explained in Section~\ref{nll_paragraph},  for each sample $x$,  we mask some components of $\latentcont^0(x)$, and aim at sampling the missing components given the observed ones using the prior models. This conditional denoising process is further explained for our model in Appendix~\ref{ap:inpainting}. To illustrate this setting, we show different conditional samples for 3 images in Figure~\ref{fig:miniimagenet_prior_ours_conditional} and  Figure~\ref{fig:miniimagenet_prior_ours_conditional:topleft}  for both the PixelCNN prior and ours. In Figure~\ref{fig:miniimagenet_prior_ours_conditional},  the mask corresponds to a $9\times 9$ centered square over the $21\times 21$ feature map. In Figure~\ref{fig:miniimagenet_prior_ours_conditional:topleft},  the mask corresponds to a $9\times 9$ top left square. These figures illustrate that our diffusion model is much less sensitive to the selected masked region than PixelCNN. This may be explained by the use of our denoising function $\varepsilon_\theta$ which depends on all conditioning pixels while PixelCNN uses a hierarchy of masked convolutions to enforce
a specific conditioning order. Additional conditional sampling experiments are given in Appendix~\ref{ap:additional_visuals}.

\begin{figure}
    \centering
    \includegraphics[width=.7\linewidth]{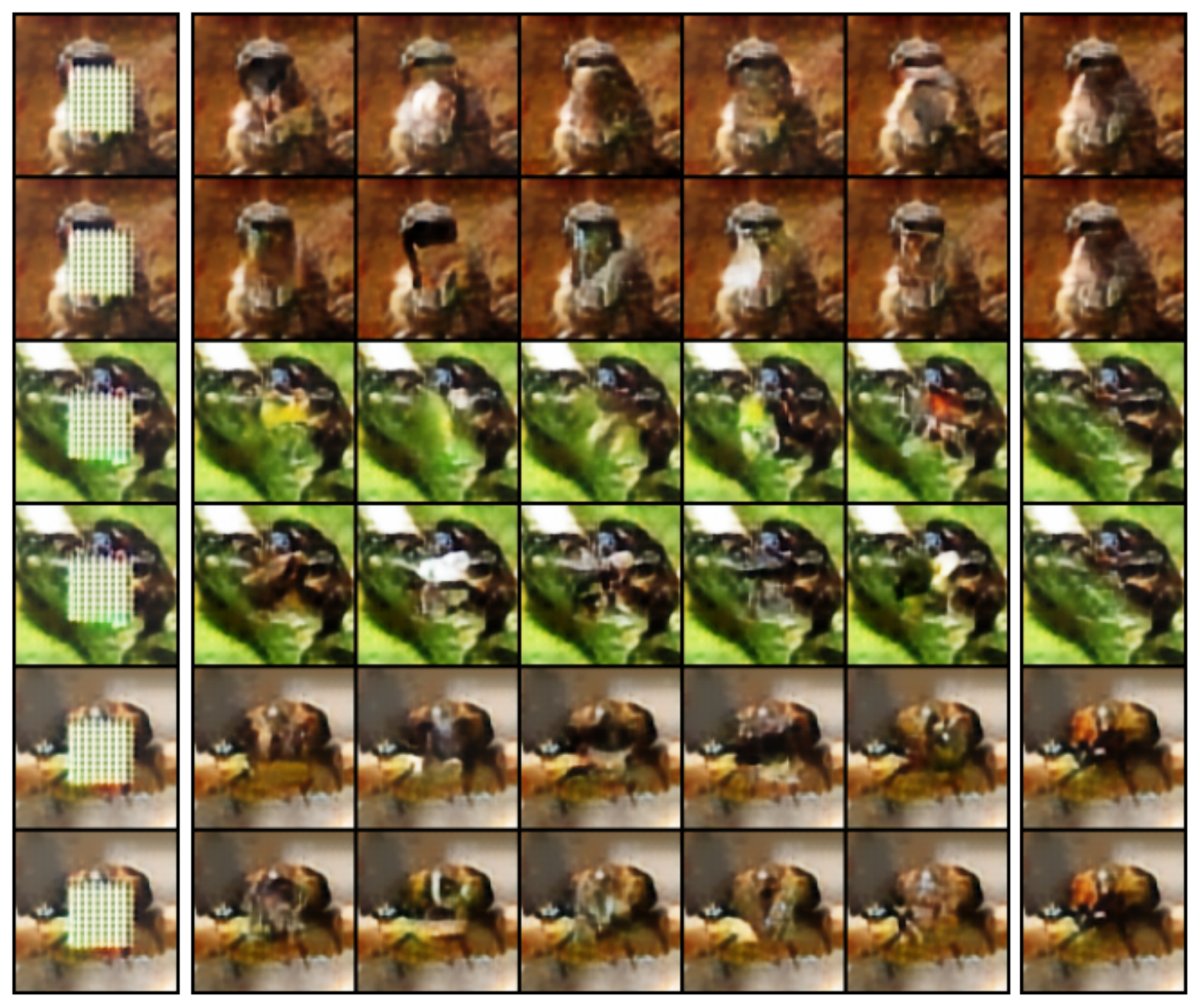}

    \caption{Conditional sampling with centered mask: for each of the 3 different images, samples from our diffusion are on top and from PixelCNN on the bottom. For each row: the image on the left is the VQVAE masked reconstruction, the image on the right is the full VQ-VAE reconstruction. Images in-between are independent conditional samples from the models.}
    \label{fig:miniimagenet_prior_ours_conditional}
\end{figure}

\begin{figure}
    \centering
    \includegraphics[width=.7\linewidth]{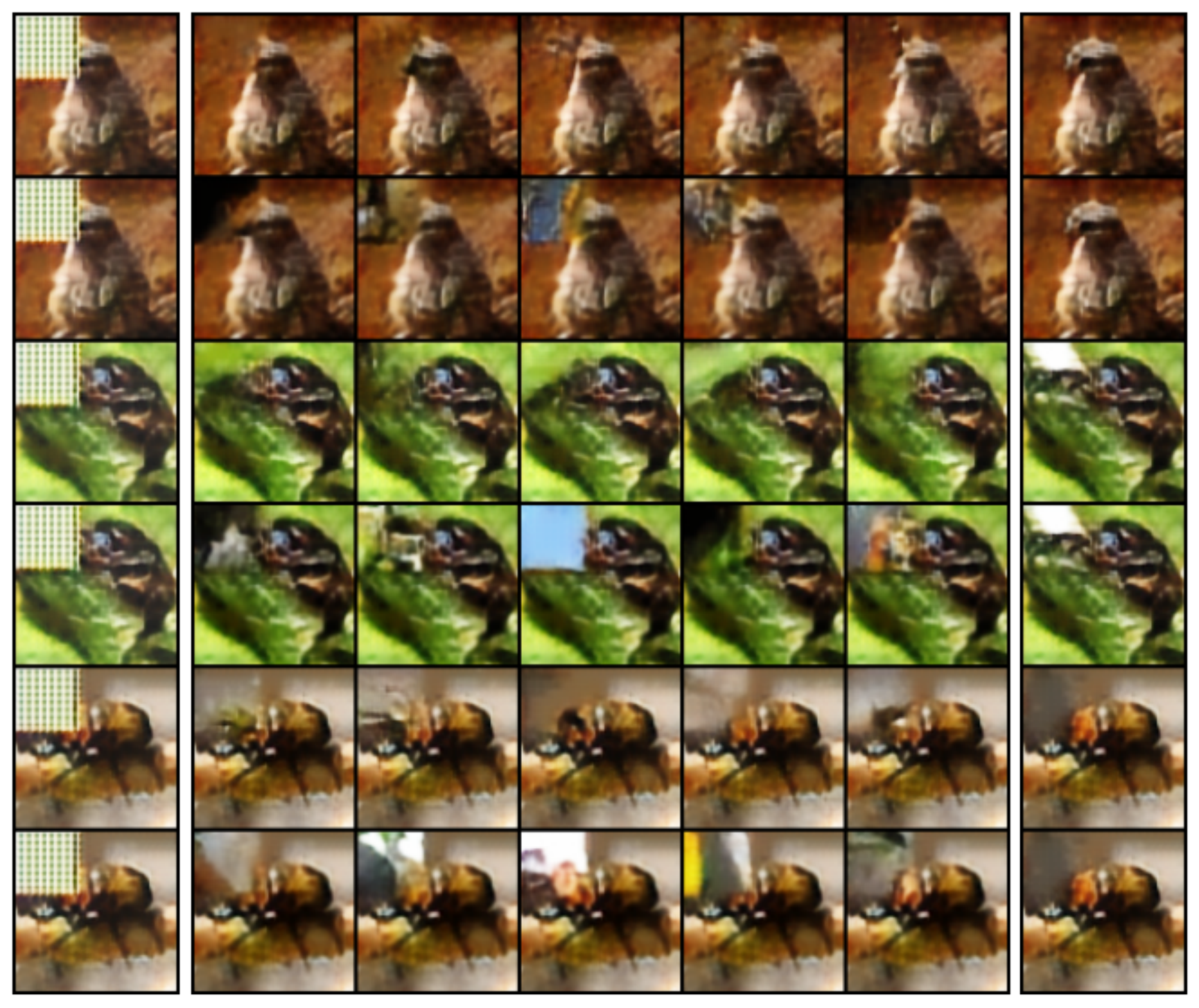}

    \caption{Conditional sampling with top left mask: for each of the 3 different images, samples from our diffusion are on top and from PixelCNN on the bottom. For each row: the image on the left is the VQVAE masked reconstruction, the image on the right is the full VQ-VAE reconstruction. Images in-between are independent conditional samples from the models.}
    \label{fig:miniimagenet_prior_ours_conditional:topleft}
\end{figure}

\paragraph{Denoising chain. }
In addition to the conditional samples, Figure~\ref{fig:miniimagenet_prior_ours_chain} shows the conditional denoising process at regularly spaced intervals, and Figure~\ref{fig:miniimagenet_prior_ours_chain_unconditional} shows unconditional denoising. Each image of the chain is generated by passing the predicted $\latentdis^t$ through the VQ-VAE decoder.

\begin{figure}
    \centering
    \includegraphics[width=.7\linewidth]{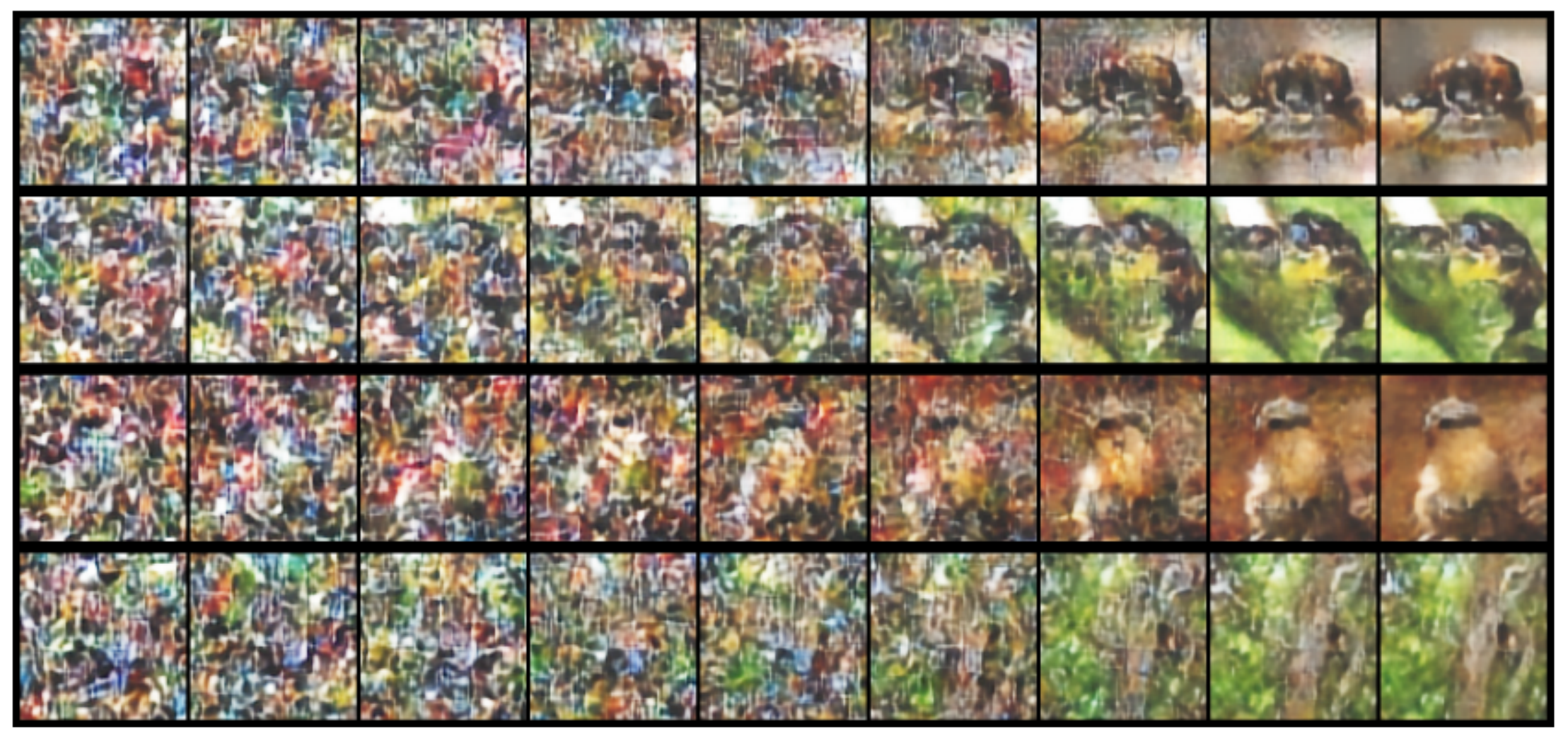}
    \caption{Sampling denoising chain from $t=500$ up to $t=0$, shown at regular intervals, conditioned on the outer part of the picture. We show only the last $500$ steps of this process, as the first $500$ steps are not visually informative. The sampling procedure is described in Appendix~\ref{ap:inpainting}.}
    \label{fig:miniimagenet_prior_ours_chain}
\end{figure}

\begin{figure}
    \centering
    \includegraphics[width=.7\linewidth]{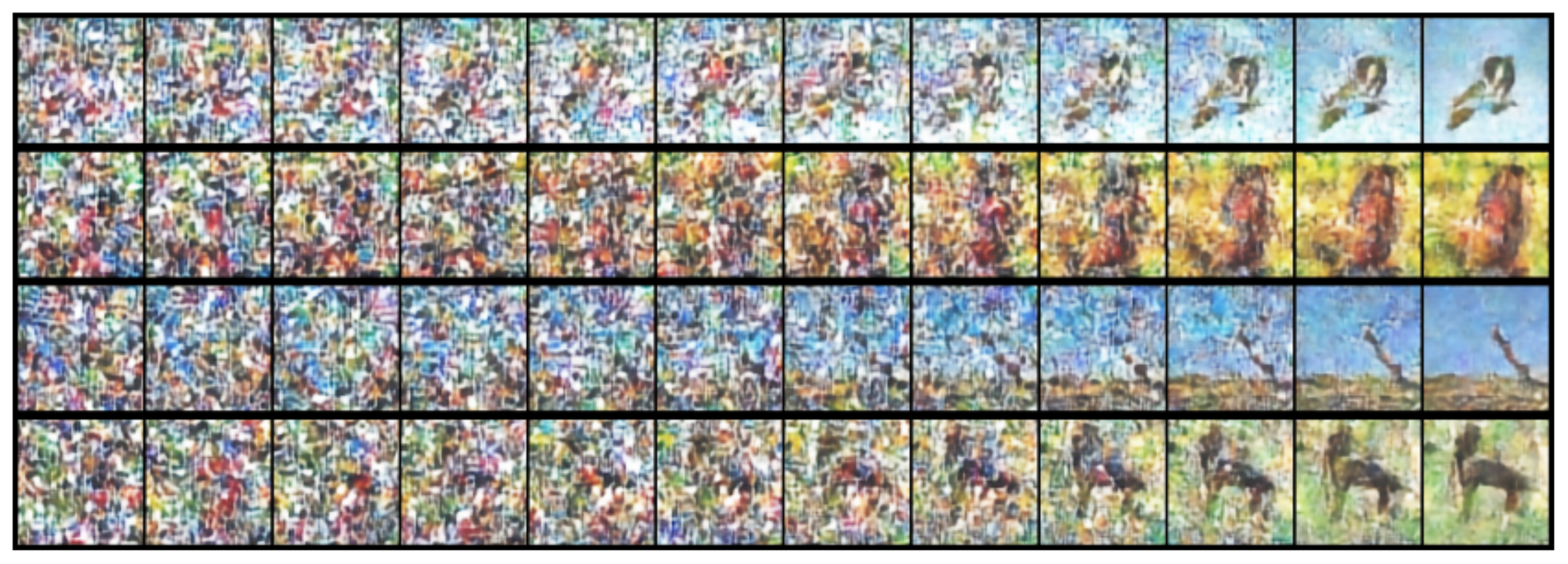}
    \caption{Sampling denoising chain from $t=500$ up to $t=0$, shown at regular intervals, unconditional. We show only the last $500$ steps of this process, as the first $500$ steps are not visually informative. The sampling procedure is described in Algorithm~\ref{alg:sample}}
    \label{fig:miniimagenet_prior_ours_chain_unconditional}
\end{figure}

\section{Conclusion}
This work introduces a new mathematical framework for VQ-VAEs which includes a diffusion probabilistic model to learn the dependencies between the continuous latent  variables alongside the encoding and decoding part of the model. We showed conceptual improvements of our model over the VQ-VAE prior, as well as first numerical results on middle scale image generation. We believe that these first numerical experiments open up many research avenues: scaling to larger models, optimal scaling of the hyperparameters, including standard tricks from other diffusion methods, studying the influence of regulazation loss for end-to-end training, etc. We hope that this framework will serve as a sound and stable foundation to derive future generative models.

\section*{Acknowledgements}
The work of Max Cohen was supported by grants from R\'egion Ile-de-France. Charles Ollion and Guillaume Quispe benefited from the support of the Chair "New Gen RetAIl" led by l’X – \'Ecole Polytechnique and the Fondation de l’\'Ecole Polytechnique, sponsored by Carrefour.


\section{Details on the loss function}
\label{ap:loss}
\begin{proof}[Proof of Lemma~\ref{lem:loss}]
By definition,
$$
\mathcal{L}(\theta,\varphi) = \mathbb{E}_{q_{\varphi}}\left[\log \frac{p_{\theta}(\latentdis^{0:T},\latentcont^{0:T},x)}{q_{\varphi}(\latentdis^{0:T},\latentcont^{0:T}| x)}\right]\,,
$$
which yields
$$
 \mathcal{L}(\theta,\varphi) = \mathbb{E}_{q_{\varphi}}\left[\log p^x_{\theta}(x|\latentdis^{0})\right]   + \mathbb{E}_{q_{\varphi}}\left[\log \frac{p^{\latentdis}_{\theta}(\latentdis^{0:T}|\latentcont^{0:T})}{q^{\latentdis}_{\varphi}(\latentdis^{0:T}|\latentcont^{0:T})}\right] +\mathbb{E}_{q_{\varphi}}\left[\log \frac{p^{\latentcont}_{\theta}(\latentcont^{0:T})}{q^{\latentcont}_{\varphi}(\latentcont^{0:T}| x)}\right]\,.
$$
The last term may be decomposed as
$$
\mathbb{E}_{q_{\varphi}}\left[\log \frac{p^{\latentcont}_{\theta}(\latentcont^{0:T})}{q^{\latentcont}_{\varphi}(\latentcont^{0:T}| x)}\right] = \mathbb{E}_{q_{\varphi}}\left[\log p^{\latentcont}_{\theta,T}(\latentcont^{T})\right] + \sum_{t=1}^T \mathbb{E}_{q_{\varphi}}\left[\log \frac{p^{\latentcont}_{\theta,t-1|t}(\latentcont^{t-1}|\latentcont^{t})}{q^{\latentcont}_{\varphi,t|t-1}(\latentcont^{t}|\latentcont^{t-1})}\right]
$$
and
$$
\mathbb{E}_{q_{\varphi}}\left[\log \frac{p^{\latentcont}_{\theta}(\latentcont^{0:T})}{q^{\latentcont}_{\varphi}(\latentcont^{0:T}| x)}\right] = \mathbb{E}_{q_{\varphi}}\left[\log p^{\latentcont}_{\theta,T}(\latentcont^{T})\right] +\mathbb{E}_{q_{\varphi}}\left[\log \frac{p^{\latentcont}_{\theta,0|1}(\latentcont^{0}|\latentcont^{1})}{q^{\latentcont}_{\varphi,1|0}(\latentcont^{1}|\latentcont^{0})}\right] + \sum_{t=2}^T \mathbb{E}_{q_{\varphi}}\left[\log \frac{p^{\latentcont}_{\theta,t-1|t}(\latentcont^{t-1}|\latentcont^{t})}{q^{\latentcont}_{\varphi,t|t-1}(\latentcont^{t}|\latentcont^{t-1})}\right]\,.
$$
By \eqref{eq:markov:bridge},
\begin{multline*}
\mathbb{E}_{q_{\varphi}}\left[\log \frac{p^{\latentcont}_{\theta}(\latentcont^{0:T})}{q^{\latentcont}_{\varphi}(\latentcont^{0:T}| x)}\right] = \mathbb{E}_{q_{\varphi}}\left[\log \frac{p^{\latentcont}_{\theta,T}(\latentcont^{T})}{q^{\latentcont}_{\varphi,T|0}(\latentcont^{T}|\latentcont^{0})}\right] + \sum_{t=2}^T \mathbb{E}_{q_{\varphi}}\left[\log \frac{p^{\latentcont}_{\theta,t-1|t}(\latentcont^{t-1}|\latentcont^{t})}{q^{\latentcont}_{\varphi,t-1|0,t}(\latentcont^{t-1}|\latentcont^{0},\latentcont^{t})}\right] \\+ \mathbb{E}_{q_{\varphi}}\left[\log p^{\latentcont}_{\theta,0|1}(\latentcont^{0}|\latentcont^{1})\right]\,,
\end{multline*}
which concludes the proof.
\end{proof}

\section{Inpainting diffusion sampling}
\label{ap:inpainting}

We consider the case in which we know a sub-part of the picture $\overline{X}$, and want to predict the complementary pixels $\underline{X}$. Knowing the corresponding $n$ latent vectors $\overline{\latentcont}^0$ which result from $\underline{X}$ through the encoder, we sample $N - n$ $\underline{\latentcont}^T$ from the uninformative distribution $\underline{\latentcont}^T \sim \mathcal{N}(0, (2\vartheta)^{-1}\eta^2 \mathbf{I}_{d \times (N - n)})$. In order to produce the chain of samples $\latentcont^{t-1}$ from $\latentcont^t$ we then follow the following procedure.
\begin{itemize}
\item $\underline{\latentcont}^{t-1}$ is predicted from $\latentcont^t$ using the neural network predictor, similar to the unconditioned case. 
\item Sample $\overline{\latentcont}^{t-1}$ using the forward bridge noising process.
\end{itemize}

\section{Additional regularisation considerations}
\label{ap:reg}

We consider here details about the parameterisation of $p_{\theta}^{\latentdis}(\latentdis^t|\latentcont^t)$ and $q_{\varphi}^{\latentdis}(\latentdis^t|\latentcont^t)$ in order to compute $\mathcal{L}^{reg}_t(\theta,\varphi)$.
Using the Gumbel-Softmax formulation provides an efficient and differentiable parameterisation.
\begin{align*}
p_{\theta,t}^{\latentdis}(\latentdis^t = \cdot|\latentcont^t) &= \mathrm{Softmax}\{(-\|\latentcont - \embed_k\|^2_2 + G_k )/\tau_t\}_{1\leqslant k \leqslant K}\,,\\
q_{\varphi,t}(\latentdis^t = \cdot|\latentcont^t) &= \mathrm{Softmax}\{(-\|\latentcont - \embed_k\|^2_2 + \tilde G_k )/\tau\}_{1\leqslant k \leqslant K}\,,
\end{align*}
 where $\{(G_k,\tilde G_k)\}_{1\leqslant k \leqslant K}$ are i.i.d. with distribution $\mathrm{Gumbel}(0,1)$, $\tau>0$, and $\{\tau_t\}_{0\leqslant t \leqslant T}$ are positive  time-dependent scaling parameters. Then, up to the additive normalizing terms,
\begin{align*}
\mathcal{L}^{reg}_t(\theta,\varphi) = \mathbb{E}_{q_{\varphi}}\left[\log \frac{p_{\theta,t}^{\latentdis}(\latentdis^{t}|\latentcont^{t})}{q_{\varphi,t}^{\latentdis}(\latentdis^{t}|\latentcont^{t})}\right] &= \left(-\frac{1}{\tau_t} + \frac{1}{\tau}\right)\|\latentcont^t - \widehat{\latentdis^t}\|_2^2   - \frac{\tilde G_k}{\tau} + \frac{G_k}{\tau_t}\,,
\end{align*}
where $\widehat{\latentdis^t}\sim q_{\varphi,t}^{\latentdis}(\latentdis^{t}|\latentcont^{t})$. Considering only the first term which depend on $\latentcont^t$ and produce non-zero gradients, we get:
$$
\mathcal{L}^{reg}_t(\theta,\varphi) = \gamma_t \|\latentcont^t - \widehat{\latentdis^t}\|_2^2
$$
where $\gamma_t = -1/\tau_t + 1/\tau$ drives the behavior of the regulariser. By choosing is $\gamma_t$ negative for large $t$, the regulariser pushes the codebooks away from $\latentcont^t$, which prevents too early specialization, or matching of codebooks with noise, as $\latentcont^{t \approx T}$ is close to the uninformative distribution. Finally, for small $t$, choosing $\gamma_t$ positive helps matching codebooks with $\latentcont$ when the corruption is small. In practice $\tau=1$ and a simple schedule from $10$ to $0.1$ for $\tau_t$ was considered in this work.

\section{Neural Networks}
\label{ap:networks}
For $\varepsilon_\theta(\latentcont^t ,t)$, we use a U-net like architecture similar to the one mentioned in \cite{ho2020denoising}. It consists of a deep convolutional neural network with 57M parameters, which is slightly below the PixelCNN architecture (95.8M parameters). The VQ-VAE encoder / decoders are also deep convolutional networks totalling 65M parameters.

\section{Toy Example Appendix}
\label{ap:additionaltoy}

\paragraph{Parameterisation}

We consider a neural network to model  $\varepsilon_\theta(\latentcont^t ,t)$. The network shown in Figure  \ref{ap:fig:toynetwork} consists of a time embedding similar to \cite{ho2020denoising}, as well as a few linear or 1D-convolutional layers, totalling around $5000$ parameters.

\begin{figure}[h!]
    \centering
    \includegraphics[scale=2.0]{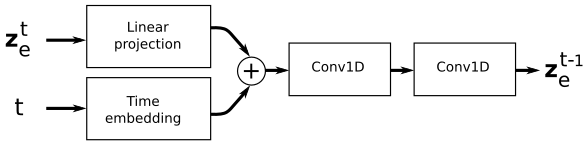}
    \caption{Graphical representation of the neural network used for the toy dataset.}
    \label{ap:fig:toynetwork}
\end{figure}

For the parameterisation of the quantization part, we choose $p_{\theta,t}^{\latentdis}(\latentdis^{t}=\embed_j|\latentcont^{t}) = \mathrm{Softmax}_{1\leq k \leq K}\{-\|\latentcont - \embed_k\|_2\}_j$, and the same parameterisation for $q_{\varphi,t}^{\latentdis}(\latentdis^{t}|\latentcont^{t})$. Therefore our loss simplifies to:
$$
\mathcal{L}(\theta,\varphi) = \mathbb{E}_{q_{\varphi}}\left[\log p^x_{\theta}(x|\latentdis^{0})\right] +  \mathcal{L}_t(\theta,\varphi)\,,
$$
where $t$ is sampled uniformly in $\{0,\ldots,T\}$.

\paragraph{Discrete samples during diffusion process}

\begin{table}[h!]
    
    \centering
    \begin{tabular}{l|c}
        t & NN sequence \\
        \hline
        50& (0, 7, 3, 6, 2)\\
        40& (6, 5, 5, 5, 3)\\
        30& (5, 5, 5, 4, 2)\\
        20& (6, 6, 5, 4, 3)\\
        10& (5, 6, 5, 4, 3)\\
        0& (5, 6, 5, 4, 3)
    \end{tabular}
    \caption{\label{ap:tab:discretetoy} Discrete samples during diffusion process. The discrete sequence is obtained by computing the nearest neighbour centroid $\mu_j$ for each $X^t_s$. At $t=0$, $X^0$ is sampled from a centered Gaussian distribution with small covariance matrix $(2\vartheta)^{-1}\eta^2\mathbf{I}_{2\times 5}$, resulting in a uniform discrete sequence, as all centroids have a similar unit norm.}

\end{table}

Discrete sequences corresponding to the denoising diffusion process shown in Figure \ref{fig:noisedenoise} are shown in Table~\ref{ap:tab:discretetoy}.

\paragraph{End-to-end training}
\label{ap:end2end}
In order to train the codebooks alongside the diffusion process, we need to backpropagate the gradient of the likelihood of the data $\latentcont$ given a $\latentcont^0$ reconstructed by the diffusion process (corresponding to $\mathcal{L}^{rec}(\theta,\varphi)$). We use the Gumbel-Softmax parameterisation in order to obtain a differentiable process and update the codebooks $\embed_j$.

In this toy example, the use of the third part of the loss $\sum_{t=0}^T \mathcal{L}^{reg}_t(\theta,\varphi)$ is not mandatory as we obtain good results with $\mathcal{L}^{reg}_t(\theta,\varphi) = 0$, which means parametrising $p_{\theta,t}^{\latentdis}(\latentdis^t|\latentcont^t) = q_{\varphi,t}^{\latentdis}(\latentdis^t|\latentcont^t)$. However we noticed that $\mathcal{L}^{reg}_t(\theta,\varphi)$ is useful to improve the learning of the codebooks. If we choose $\gamma_t$ to be decreasing with time $t$, we have the following. When $t$ is low, the denoising process is almost over, $\mathcal{L}^{reg}_t(\theta,\varphi)$ pushes $\latentcont$ and the selected $\latentdis$ to close together: $\|\latentcont\| \sim 1$, then $\|\latentcont^t\|$ will be likely near a specific $\embed_j$ and far from the others; therefore only a single codebook is selected and receives gradient. When $t$ is high, $\|\latentcont^t\| \sim 0$ and the Gumbel-Softmax makes it so that all codebooks are equidistant from $\|\latentcont^t\|$ and receive non-zero gradient. This naturally solves training problem associated with dead codebooks in VQ-VAEs. Joint training of the denoising and codebooks yield excellent codebook positionning as shown in Figure \ref{ap:codebooks}.

\begin{figure}[h!]
    \centering
    \includegraphics[scale=0.4]{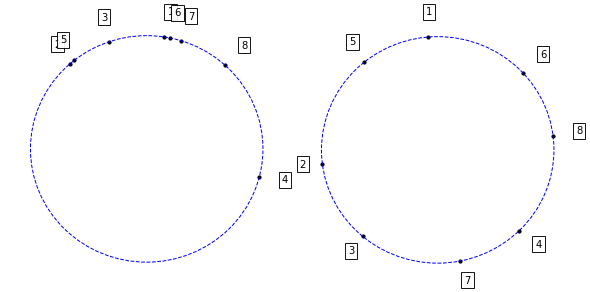}
    \caption{Left, initial random codebooks positions. Right, after training, position of codebook vectors. Note that the codebook indexes do not match the indexes of the Gaussians, the model learnt to make the associations between neighboring centroids in a different order.}
    \label{ap:codebooks}
\end{figure}

\paragraph{Toy Diffusion inpainting}

\label{ap:toyinpainting}
We consider a case in which we want to reconstruct an $x$ while we only know one (or a few) dimensions, and sample the others. Consider that $x$ is generated using a sequence $q= (q_1,q_2,q_",q_4,q_5)$ where the last one if fixed $q_1 = 0, q_5 = 4$. Then, knowing $q_1, q_5$, we sample $q_2,q_3,q_4$, as shown in Figure \ref{fig:fixeddim}.
\begin{figure}[h!]
    \centering
    \includegraphics[scale=0.5]{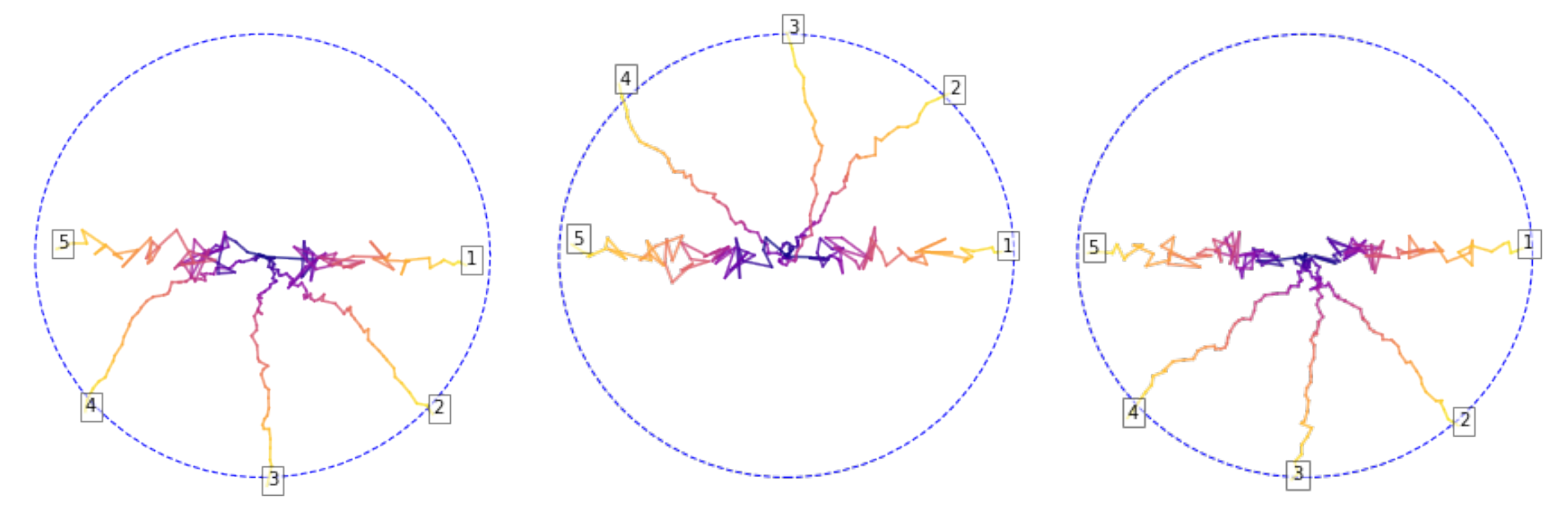}
    \caption{Three independent sampling of $X$ using a trained diffusion bridge, with fixed $q_1 = 0, q_5 = 4$. The three corresponding sequences are $(0,7,6,5,4)$, $(0,1,2,3,4)$, $(0,7,6,5,4)$ all valid sequences.}
    \label{fig:fixeddim}
\end{figure}

\clearpage
\newpage
 
\section{Additional visuals}\label{ap:additional_visuals}

\subsection{Cifar}

\begin{figure}[!htb]
    \centering
    \includegraphics[width=\linewidth]{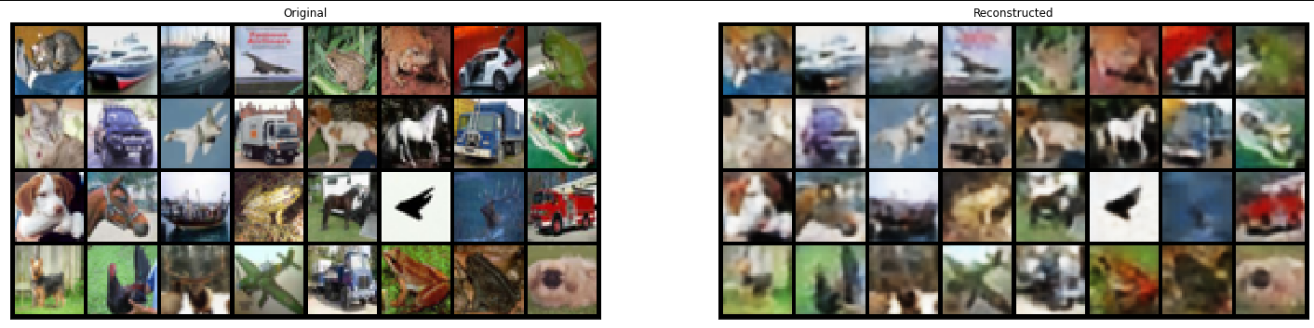}
    \caption{Reconstruction of the VQVAE model used in the following benchmarks.}
    \label{fig:cifar_vqvae}
\end{figure}

\begin{figure}[!htb]
    \centering
    \begin{minipage}[b]{0.4\textwidth}
    \includegraphics[width=1.\linewidth]{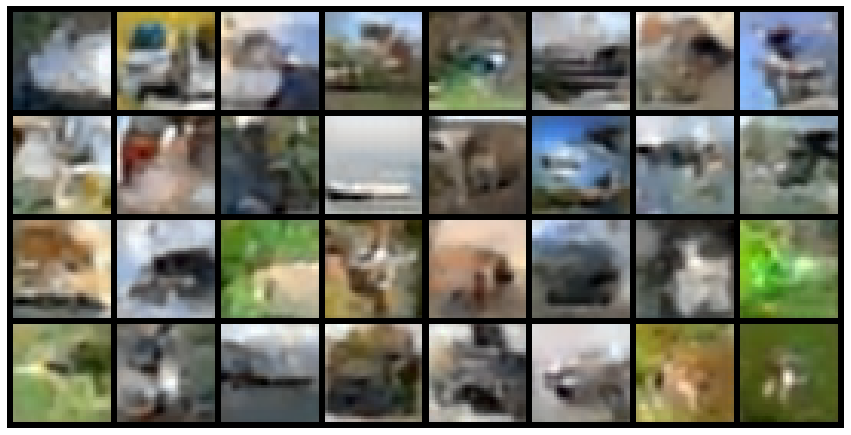}
    \end{minipage}
    \hfill
    \begin{minipage}[b]{0.4\textwidth}
    \includegraphics[width=1.\linewidth]{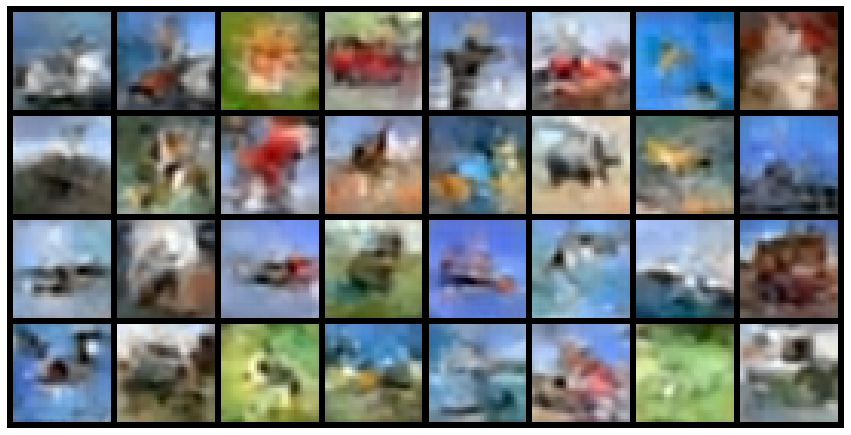}
    \end{minipage}
        \caption{Samples from the PixelCNN prior (left) and from our diffusion prior (right) on CIFAR10.}

    \label{fig:cifar_priors}
\end{figure}

\subsection{MiniImageNet}
\begin{figure}[!htb]
    \centering
    \includegraphics[width=0.9\linewidth]{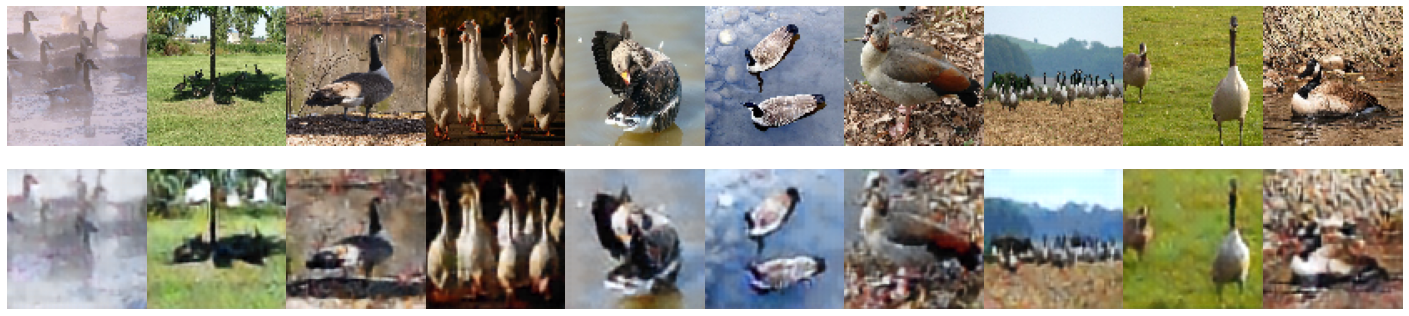}
    \caption{Reconstruction of the trained VQ-VAE on the \textit{mini}ImageNet dataset. Original images are encoded, discretised, and decoded.}
    \label{fig:miniimagenet_vqvae}
\end{figure}

\begin{figure}
    \centering
    \includegraphics[width=1.\linewidth]{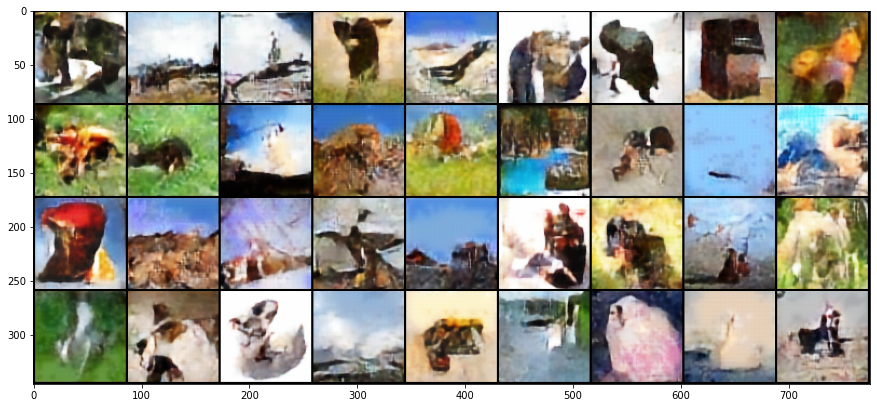}
    \caption{Samples from our model for the miniimagenet dataset}
    \label{fig:miniimagenet_prior_ours2}
\end{figure}

\begin{figure}
    \centering
    \includegraphics[width=1.\linewidth]{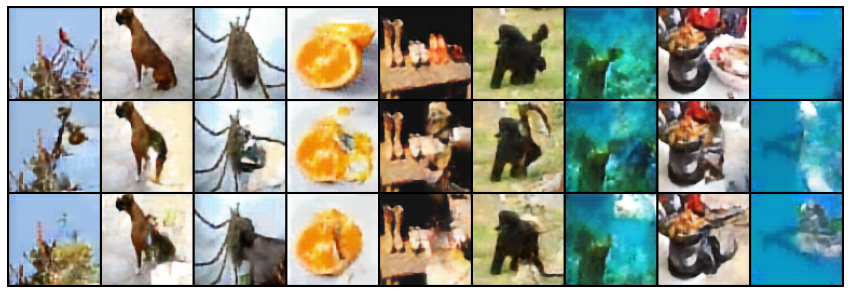}
    \includegraphics[width=1.\linewidth]{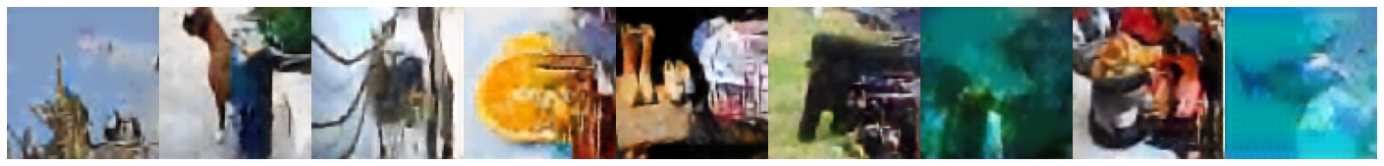}
    \includegraphics[width=1.\linewidth]{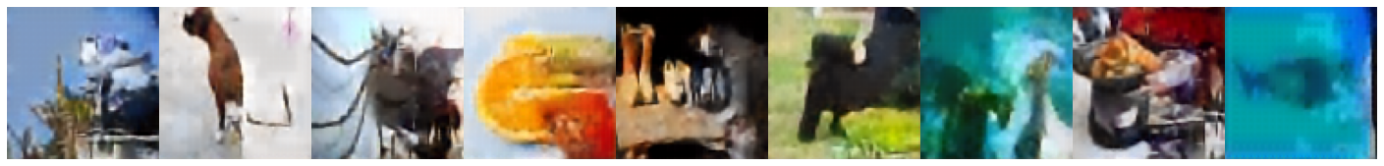}
    \caption{Conditional sampling: Top: reconstructions from the vqvae of originals images, Middle: conditional sampling with the left side of the image as condition, for our model. Bottom 1 and 2: conditional sampling in the same context with the PixelCNN prior.}
    \label{fig:miniimagenet_prior_ours_conditional2}
\end{figure}

\begin{figure}
    \centering
    \includegraphics[width=.8\linewidth]{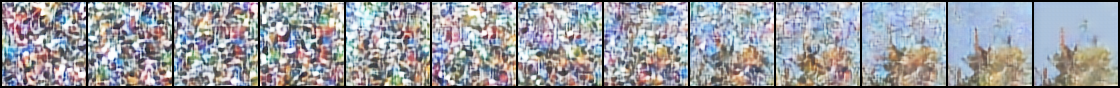}
    \includegraphics[width=.8\linewidth]{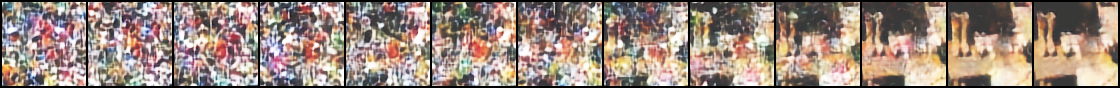}
    \includegraphics[width=.8\linewidth]{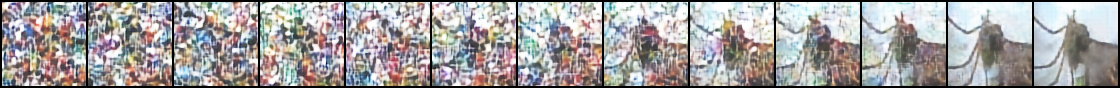}
    \includegraphics[width=.8\linewidth]{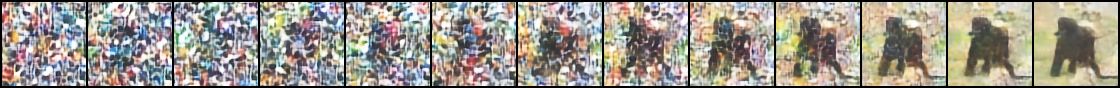}
    \includegraphics[width=.8\linewidth]{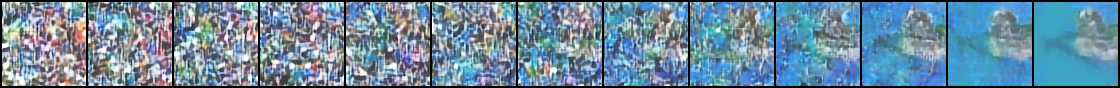}
    \includegraphics[width=.8\linewidth]{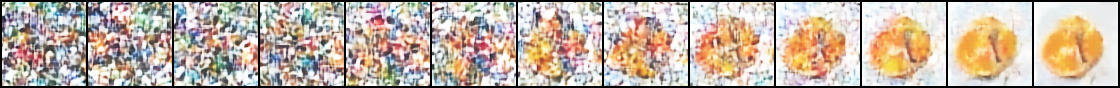}
    \caption{Sampling denoising chain from up to $t=0$, shown at regular intervals, conditioned on the left part of the picture. The sampling procedure is described in Appendix~\ref{ap:inpainting}.}
    \label{fig:miniimagenet_prior_ours_chain2}
\end{figure}

\begin{figure}
    \centering
    \begin{minipage}[b]{0.45\textwidth}
        \includegraphics[width=1.\linewidth]{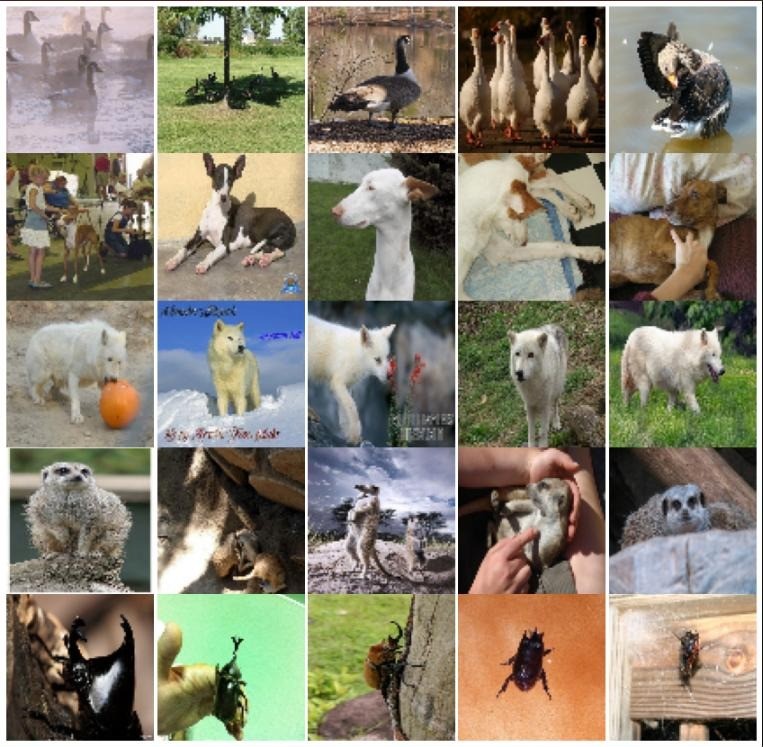}
    \end{minipage}
    \hfill
    \begin{minipage}[b]{0.45\textwidth}

    \includegraphics[width=1.\linewidth]{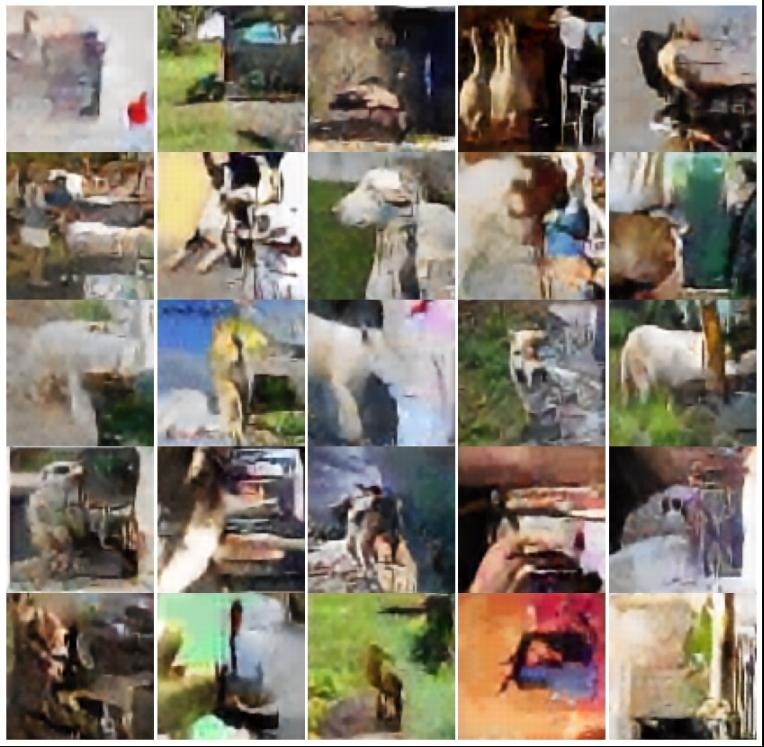}
        \end{minipage}
    \caption{Conditional sampling with the PixelCNN prior. \textbf{Left}: original images, \textbf{Right}: conditional sampling with the left side of the image as condition. Each row represents a class of the validation set of the \textit{mini}ImageNet dataset.}
    \label{fig:miniimagenet_prior_pixelcnn_conditional}

\end{figure}

\clearpage

\bibliographystyle{apalike}
\bibliography{vqvae}

\begin{thebibliography}{}

\bibitem[Austin et~al., 2021]{austin2021structured}
Austin, J., Johnson, D.~D., Ho, J., Tarlow, D., and van~den Berg, R. (2021).
\newblock Structured denoising diffusion models in discrete state-spaces.
\newblock {\em Advances in Neural Information Processing Systems}.

\bibitem[Beskos et~al., 2008]{beskos2008mcmc}
Beskos, A., Roberts, G., Stuart, A., and Voss, J. (2008).
\newblock Mcmc methods for diffusion bridges.
\newblock {\em Stochastics and Dynamics}, 8(03):319--350.

\bibitem[Bladt et~al., 2016]{bladt2016simulation}
Bladt, M., Finch, S., and S{\o}rensen, M. (2016).
\newblock Simulation of multivariate diffusion bridges.
\newblock {\em Journal of the Royal Statistical Society: Series B: Statistical
  Methodology}, pages 343--369.

\bibitem[Chen et~al., 2018]{chen2018pixelsnail}
Chen, X., Mishra, N., Rohaninejad, M., and Abbeel, P. (2018).
\newblock Pixelsnail: An improved autoregressive generative model.
\newblock In {\em International Conference on Machine Learning}, pages
  864--872. PMLR.

\bibitem[De~Bortoli et~al., 2021]{de2021simulating}
De~Bortoli, V., Doucet, A., Heng, J., and Thornton, J. (2021).
\newblock Simulating diffusion bridges with score matching.
\newblock {\em arXiv preprint arXiv:2111.07243}.

\bibitem[Esser et~al., 2021]{esser2021taming}
Esser, P., Rombach, R., and Ommer, B. (2021).
\newblock Taming transformers for high-resolution image synthesis.
\newblock {\em Proceedings of the IEEE/CVF Conference on Computer Vision and
  Pattern Recognition}, pages 12873--12883.

\bibitem[Gu et~al., 2021]{gu2021vector}
Gu, S., Chen, D., Bao, J., Wen, F., Zhang, B., Chen, D., Yuan, L., and Guo, B.
  (2021).
\newblock Vector quantized diffusion model for text-to-image synthesis.
\newblock {\em arXiv preprint}.

\bibitem[Ho et~al., 2020]{ho2020denoising}
Ho, J., Jain, A., and Abbeel, P. (2020).
\newblock Denoising diffusion probabilistic models.
\newblock {\em Advances in Neural Information Processing Systems (NeurIPS
  2021)}, 34.

\bibitem[Hoogeboom et~al., 2021]{hoogeboom2021argmax}
Hoogeboom, E., Nielsen, D., Jaini, P., Forr{\'e}, P., and Welling, M. (2021).
\newblock Argmax flows and multinomial diffusion: Learning categorical
  distributions.
\newblock {\em Advances in Neural Information Processing Systems (NeurIPS
  2021)}, 34.

\bibitem[Lin et~al., 2010]{lin2010generating}
Lin, M., Chen, R., and Mykland, P. (2010).
\newblock On generating monte carlo samples of continuous diffusion bridges.
\newblock {\em Journal of the American Statistical Association},
  105(490):820--838.

\bibitem[Mittal et~al., 2021]{mittal2021symbolic}
Mittal, G., Engel, J., Hawthorne, C., and Simon, I. (2021).
\newblock Symbolic music generation with diffusion models.
\newblock {\em arXiv preprint arXiv:2103.16091}.

\bibitem[Oord et~al., 2016]{oord2016wavenet}
Oord, A. v.~d., Dieleman, S., Zen, H., Simonyan, K., Vinyals, O., Graves, A.,
  Kalchbrenner, N., Senior, A., and Kavukcuoglu, K. (2016).
\newblock Wavenet: A generative model for raw audio.
\newblock {\em arXiv preprint arXiv:1609.03499}.

\bibitem[Oord et~al., 2017]{oord2017neural}
Oord, A. v.~d., Vinyals, O., and Kavukcuoglu, K. (2017).
\newblock Neural discrete representation learning.
\newblock {\em Advances in neural information processing systems (NeurIPS
  2017)}.

\bibitem[Ramesh et~al., 2021]{ramesh2021zero}
Ramesh, A., Pavlov, M., Goh, G., Gray, S., Voss, C., Radford, A., Chen, M., and
  Sutskever, I. (2021).
\newblock Zero-shot text-to-image generation.
\newblock 139:8821--8831.

\bibitem[Razavi et~al., 2019]{razavi2019generating}
Razavi, A., van~den Oord, A., and Vinyals, O. (2019).
\newblock Generating diverse high-fidelity images with vq-vae-2.
\newblock In {\em Advances in neural information processing systems (NeurIPS
  2019)}, pages 14866--14876.

\bibitem[Salimans et~al., 2017]{salimans2017pixelcnn}
Salimans, T., Karpathy, A., Chen, X., and Kingma, D.~P. (2017).
\newblock Pixelcnn++: Improving the pixelcnn with discretized logistic mixture
  likelihood and other modifications.

\bibitem[Sohl-Dickstein et~al., 2015]{sohldickstein2015deep}
Sohl-Dickstein, J., Weiss, E., Maheswaranathan, N., and Ganguli, S. (2015).
\newblock Deep unsupervised learning using nonequilibrium thermodynamics.
\newblock 37:2256--2265.

\bibitem[Song et~al., 2021]{song2021denoising}
Song, J., Meng, C., and Ermon, S. (2021).
\newblock Denoising diffusion implicit models.

\bibitem[Song and Ermon, 2019]{song2019generative}
Song, Y. and Ermon, S. (2019).
\newblock Generative modeling by estimating gradients of the data distribution.
\newblock 32.

\bibitem[Vahdat et~al., 2021]{vahdat2021score}
Vahdat, A., Kreis, K., and Kautz, J. (2021).
\newblock Score-based generative modeling in latent space.

\bibitem[van~den Oord et~al., 2016]{oord2016conditional}
van~den Oord, A., Kalchbrenner, N., Vinyals, O., Espeholt, L., Graves, A., and
  Kavukcuoglu, K. (2016).
\newblock Conditional image generation with pixelcnn decoders.

\bibitem[Van~Oord et~al., 2016]{van2016pixel}
Van~Oord, A., Kalchbrenner, N., and Kavukcuoglu, K. (2016).
\newblock Pixel recurrent neural networks.
\newblock In {\em International Conference on Machine Learning}, pages
  1747--1756. PMLR.

\bibitem[Vinyals et~al., 2016]{Vinyals2016MatchingNF}
Vinyals, O., Blundell, C., Lillicrap, T., kavukcuoglu, k., and Wierstra, D.
  (2016).
\newblock Matching networks for one shot learning.
\newblock In Lee, D., Sugiyama, M., Luxburg, U., Guyon, I., and Garnett, R.,
  editors, {\em Advances in Neural Information Processing Systems}, volume~29.
  Curran Associates, Inc.

\bibitem[Willetts et~al., 2021]{willetts:2021}
Willetts, M., Miscouridou, X., Roberts, S., and Holmes, C. (2021).
\newblock Relaxed-responsibility hierarchical discrete {VAE}s.
\newblock {\em ArXiv:2007.07307}.

\end{thebibliography}

\end{document}